%% file: manuscript_report.tex
\documentclass[letterpaper, 10 pt, conference]{ieeeconf}  

\IEEEoverridecommandlockouts                              

\input{packages}
\input{commands}

\input{notation}

\title{\LARGE \bf
From Low to High Order Motion Planners: \\ Safe Robot Navigation
using Motion Prediction and Reference Governor
}

\author{Aykut \.{I}\c{s}leyen and Nathan van de Wouw and \"{O}m\"{u}r Arslan
\thanks{The authors are with the Department of Mechanical Engineering, Eindhoven University of Technology, P.O. Box 513, 5600 MB Eindhoven, The Netherlands. The authors are also affiliated with the Eindhoven AI Systems Institute. Emails:  \{a.isleyen, n.v.d.wouw, o.arslan\}@tue.nl}%
}

\begin{document}

\maketitle
\thispagestyle{empty}
\pagestyle{empty}

\begin{abstract}
Safe navigation around obstacles is a fundamental challenge for highly dynamic robots.
The state-of-the-art approach for adapting simple reference path planners to complex robot dynamics using trajectory optimization and tracking control is brittle and requires significant replanning cycles.
In this paper, we introduce a novel feedback motion planning framework that extends the applicability of low-order  (e.g. position-/velocity-controlled) reference motion planners to high-order (e.g., acceleration-/jerk-controlled) robot models using motion prediction and reference governors. 
We use predicted robot motion range for safety assessment and establish a bidirectional interface between high-level planning and low-level control via a reference governor.
We describe the generic fundamental building blocks of our feedback motion planning framework and give specific example constructions for motion control, prediction, and reference planning.
We prove the correctness of our planning framework and demonstrate its performance in numerical simulations. 
We conclude that accurate motion prediction is crucial for closing the gap between high-level planning and low-level control.   

\end{abstract}

\section{Introduction}

Safe and smooth robot motion is fundamental for many autonomous systems, and people interacting with them.
Kinodynamic motion planning of dynamically feasible and safe trajectories is known to be computationally hard for many robotic systems \cite{donald_etal_JACM1993, lavalle_kuffner_IJRR2001} because determining safety of highly dynamic robot systems is a challenge \cite{fraichard_asama_AR2004}.
The state-of-the-art smooth motion planning approaches \cite{ravankar_etal_Sensors2018} often start with a simple reference path planner \cite{lavalle_PlanningAlgorithms2006}, and then build a dynamically feasible and safe robot trajectory (i.e., a time-parametrized path) around a reference path using trajectory optimization  \cite{richter_etal_ISRR016}.
However, such trajectory planning methods often suffer from significant replanning cycles in practice due to precise trajectory tracking requirements \cite{ding_gao_wang_shen_TRO2019, tordesillas_etal_ToR2021}.

In this paper, we propose a novel provably correct feedback motion planning framework that extends safety and navigation properties of low-order (e.g. position-/velocity-controlled) reference motion planners to high-order (e.g., acceleration-/jerk-controlled) robot dynamics using a bidirectional interface between high-level planning and low-level control based on a reference governor and safety assessment of predicted robot motion, as illustrated in \mbox{\reffig{fig.general_framework}}.

\begin{figure}[t]   
\centering
\vspace{-0.5mm}
\includegraphics[width = \columnwidth]{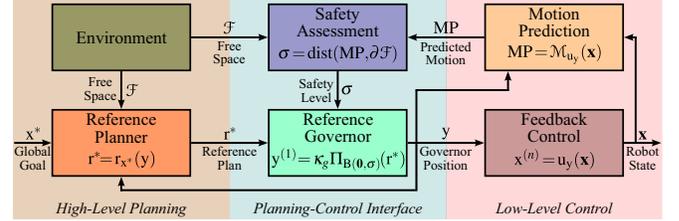}
\vspace{-4.5mm}
\caption{Safe feedback motion planning using motion prediction and reference governor enables a bidirectional safety interface between high-level planning and low-level control.}
\label{fig.general_framework}
\vspace{-5mm}
\end{figure}

\vspace{-1mm}

\subsection{Motivation and Relevant Literature}

Safe and smooth robot motion is essential for bringing robots from factories to our everyday lives, because  
autonomous robot operation around people requires safety guarantees, and jerky robot motion does not only cause system failures and malfunctions, but also causes discomfort for people. 
Most existing safe and smooth motion planning methods follow a two-step approach: first find a simple piecewise linear path using a standard off-the-shelf path planner \cite{lavalle_PlanningAlgorithms2006}, and then perform path smoothing and time parametrization using trajectory optimization in order to satisfy safety and system requirements \cite{liu_etal_RAL2017, richter_etal_ISRR016, ding_etal_RAL2019}.
Unfortunately, the open-loop nature of such trajectory planning methods often requires many replanning cycles to ensure safety by closing the planning loop \cite{tordesillas_etal_ToR2021, ding_gao_wang_shen_TRO2019}. 
In this paper, we introduce a new feedback motion planning framework for safe and smooth robot motion design by establishing a continuous bidirectional safety interface between high-level reference planing and low-level motion control via motion prediction and reference governors.

Reference governors are originally designed for constrained control of dynamical systems to separately handle the issues of stability and constraint satisfaction  \cite{bemporad_TAC1998, gilbert_kolmanovsky_Automatica2002, garone_nicotra_TAC2015}. 
In robotics, reference governors are applied for safe robot navigation to separately address global navigation, safety, and stability, and safety requirements at different stages by high-level planning and low-level control \cite{arslan_koditschek_ICRA2017}.
Reference governors are successfully demonstrated for safe navigation of second-order robotic systems using total energy \cite{arslan_koditschek_ICRA2017} and directed Lyapunov ellipsoids  \cite{li_ICRA2020, li_2020}.
These existing approaches indirectly suggest that motion prediction and safety assessment is fundamental in governed robot motion design in their very specific constructions. 
In this paper we identify and describe generic fundamental building blocks for governed robot motion design, provide new example constructions, and systematically investigate the role of motion prediction and system order on robot motion.
At a more conceptual level, our newly introduced notion of motion range prediction allows one to effectively determine (inevitable) collision states of  high-order robotic systems which is a known bottleneck in kinodynamic motion planning  \cite{fraichard_asama_AR2004}.

\subsection{Contributions and Organization of the Paper}

This paper introduces a new feedback motion planning framework that closes the gap between traditionally separately studied high-level planning and low-level control by establishing a bidirectional planning-control interface based on motion prediction, safety assessment, and reference governors.
In \refsec{sec.GeneralFramework}, we present fundamental motion design elements of our motion planning framework, describe its working principles, and analyze its stability and safety properties.
In \refsec{sec.ExampleMotionControlPrediction}, we provide example constructions for low-level motion control and prediction, and high-level reference planning, and demonstrate the use of Lyapunov ellipsoids and Vandermonde simplexes for robot motion prediction under PhD control. 
In \refsec{sec.NumericalSimulations}, we systematically investigate the role of motion prediction on robot motion in numerical simulations.  
We conclude in \refsec{sec.Conclusions} with a summary of our contributions and future work.

\section{From Low to High Order Motion Planners:\\ Problem Formulation}
\label{sec.ProblemFormulation}

For ease of exposition, we consider a disk-shaped robot of radius $\radius \in \R_{>0}$, centered at position $\pos \in \workspace$,  operating in a known static closed compact environment $\workspace \subseteq \R^{\dimspace}$ that is cluttered with a collection of obstacles represented by an open set $\obstspace \subset \R^{\dimspace}$, where $\dimspace \geq 2$. 
Hence, the robot's free space, denoted by $\freespace$, of collision-free positions is given by
\begin{align} \label{eq.FreeSpace}
\freespace \ldf \clist{ \pos \in \workspace \, \big| \,   \ball(\pos,\radius) \subseteq \workspace \setminus \obstspace },
\end{align}
where $\ball(\pos,\radius) \ldf \clist{ \vect{y} \in \R^{\dimspace} \big|  \norm{\vect{y} - \pos } \leq \radius} $ is the Euclidean closed ball centered at $\pos$ with radius $\radius$, and  $\norm{.}$ denotes the standard Euclidean norm for both vectors and matrices. 
To ensure global navigation between any start and goal positions in $\freespace$, we assume that the free space $\freespace$ is path-connected.

We also assume that the robot motion is fully actuated and described by an $\order^{\text{th}}$-order differential equation of the form\,\footnote{For example, one can consider a complex nonlinear dynamical system controlled via feedback linearization \cite{khalil_NonlinearSystems2001}.}  
\begin{align}
\pos^{(\order)} =  \ctrl(\state) = \ctrl(\pos^{(0)}, \pos^{(1)}, \ldots, \pos^{(\order-1)}),
\end{align}
where $\pos^{(k)} = \frac{\deriv^{k} }{\deriv t ^{k}} \pos$ denotes the $k^{\text{th}}$ time derivative of the robot position $\pos$, and $\ctrl: \plist{\R^{\dimspace}}^{\order} \rightarrow \R^{\dimspace}$ is a feedback  control policy, and  $\state = (\pos^{(0)}, \pos^{(1)}, \ldots, \pos^{(\order-1)}) \in \plist{\R^{\dimspace}}^{\order}$  denotes  the robot's dynamical state consisting of its position, velocity, and so on. 
Accordingly, one can determine if the robot is in motion or not as follows:
\begin{definition}\label{def.ZeroMotion} 
(\emph{Zero Motion})
 An $\order^{\text{th}}$-order robot is said to be in the \emph{zero-motion} state iff its velocity and higher-order variables are zero, i.e., $\pos^{(k)} = \mat{0}$ for all $k = 1, \ldots, \order -1$, where $\mat{0}$ is a vector of all zeros of the appropriate size.
Hence, any robot state  $\state \in \R^{\dimspace} \times \clist{\mat{0}}^{\order-1}$ is a zero-motion state.
\end{definition} 
\noindent For example, for the second-order dynamics, zero-motion corresponds to zero-velocity.
Note that zero-motion is different than being stationary which means, in addition to the condition for zero motion, $\ctrl(\state) = \mat{0}$.

For the simple fully-actuated kinematic (position- and velocity-controlled) robot model, the global navigation problem  can be solved effectively using off-the-shelf motion planning algorithms \cite{choset_etal_PrinciplesOfRobotMotion2005, lavalle_PlanningAlgorithms2006}.
However, kinodynamic motion planning of complex high-order robotic systems is an open research challenge because collision and safety verification of highly dynamical robotic systems is hard \cite{lavalle_kuffner_IJRR2001}. 
As a new approach for closing the gap between kinematic and kinodynamic motion planning, we consider  \emph{provably-correct extensions of feedback motion planners} from simple (i.e., first-order) to complex (i.e., $n^\text{th}$-order) robot dynamics that preserve and augment stability and safety properties: 
\begin{MotionPlannerExtension}
Given a Lipschitz continuous vector field planner $\refplan_{\goal}: \refdomain \rightarrow \R^{\dimspace}$ for the first-order robot dynamics
\begin{align}\label{eq.ReferenceDynamics}
\pos^{(1)} = \refplan_{\goal}(\pos),
\end{align} 
that  asymptotically steers all robot positions in its positively invariant collision-free domain $\refdomain \subseteq \freespace$  towards a safe goal location $\goal \in \refdomain$,  
a correct extension of the reference motion planner $\refplan_{\goal}$  for  the $n^{\text{th}}$-order  dynamics is a construction of a feedback motion planner $\ctrl_{\refplan_{\goal}}: \plist{\R^{\dimspace}}^{\order} \rightarrow \R^{\dimspace}$ of the form 
\begin{align}
\pos^{(n)} = \ctrl_{\refplan_{\goal}}(\state) = \ctrl_{\refplan_{\goal}}\plist{\pos^{0}, \pos^{(1)}, \ldots, \pos^{(n-1)}},
\end{align}
that asymptotically brings all collision-free zero-motion initial robot states\,\footnote{This requirement on the initial robot states can be extended to include other collision-free robot states as defined in \refcor{cor.CollisionFreeState}.} in $\refdomain \times \clist{\mat{0}} ^{n-1}$ to the goal location $\goal$ while ensuring no collision along the way, i.e., the robot trajectory $\pos(t)$ stays in the free space $\freespace$ for all future times $t \geq 0$. 
\end{MotionPlannerExtension}

Below, we introduce a general framework with examples for provably correct extensions of feedback motion planners from low-order to high-order robot dynamics using reference governors and risk assessment of predicted robot motion.

\section{From Low to High Order Motion Planners: General Framework}
\label{sec.GeneralFramework}

In this section, we describe the generic building blocks and the working principle of our motion planning framework for extending the applicability of a first-order reference planner to the $\order^{\text{th}}$-order robot dynamics that are stabilized by a standard feedback motion control policy.
We realize a bidirectional interface between the high-level reference motion planner and the low-level feedback motion control via a reference governor system using motion range prediction and safety assessment, as illustrated in \reffig{fig.general_framework}. 
We also analyze important stability and safety properties of our general extension framework to demonstrate its correctness.

\subsection{Fundamental Motion Planning Elements}

Our motion planning framework consists of five building elements: reference motion planner, feedback motion control, motion prediction, safety assessment, and reference governor, whose specific roles and requirements are presented below.

\subsubsection{Reference Motion Planner} 
\label{sec.ReferenceMotionPlanner}

A \emph{reference motion planner} for the first-order and fully-actuated robot dynamics in \refeq{eq.ReferenceDynamics} is a Lipschitz continuous vector field planner $\refplan_{\goal}: \refdomain \rightarrow \R^{\dimspace}$ associated with a goal location $\goal \in \refdomain$ in  its collision-free domain $\refdomain \subseteq \freespace$ such that 
\begin{itemize}
\item (Positive Invariance \& Boundary Avoidance\footnote{The boundary avoidance property is required for ensuring that the robot-governor system  can asymptotically follow a reference motion planner without getting stuck along the way (see the proof of \refprop{prop.Stability}).
}) $\refplan_{\goal}$ is inward pointing on the boundary $\partial \refdomain$ of its domain $\refdomain$.
\item (Global Asymptotic Stability) $\refplan_{\goal}$ has a unique stable point at $\goal$ whose domain of attraction includes $\refdomain$.
\end{itemize}
In other words, for the velocity-controlled robot model, the reference motion planner $\refplan_{\goal}$ asymptotically brings all robot positions in its positively-invariant domain $\refdomain$  towards the  goal $\goal$ while avoiding its boundary and so collisions \cite{arslan_koditschek_ICRA2017}.

\subsubsection{Feedback Motion Control}\label{sec.FeedbackMotionControl}

A feedback motion control policy $\ctrl_{\ctrlgoal}:\plist{\R^{\dimspace}}^{\order} \rightarrow \R^{\dimspace}$ for the $\order^{\text{th}}$-order robot dynamics 
\begin{align}\label{eq.FeedbackMotionControl}
\pos^{(\order)} = \ctrl_{\ctrlgoal}(\state) = \ctrl_{\ctrlgoal}(\pos^{(0)}, \pos^{(1)}, \ldots, \pos^{(\order-1)}),
\end{align}
is a Lipschitz continuous controller that is parametrized with a desired robot position $\ctrlgoal \in \R^{\dimspace}$ at which the closed-loop system is  globally asymptotically stable  with zero-motion (\refdef{def.ZeroMotion}), i.e.,  $\ctrl_{\ctrlgoal}(\ctrlgoal, \mat{0}, \ldots, \mat{0})  = \mat{0}$.
As a controller choice, one can consider any  standard feedback motion controller as long as its motion range can be accurately bounded in terms of the robot state $\state$ and the goal $\ctrlgoal$,  as discussed next.

\subsubsection{Motion Range Prediction}\label{sec.motion_range_prediction}

For the $\order^{\text{th}}$-order robot model that moves  towards a given goal $\ctrlgoal \in \R^{\dimspace}$  under the feedback motion control $\ctrl_{\ctrlgoal}$, a \textit{motion range prediction}, denoted by $\motionrange_{\ctrl_{\ctrlgoal}}(\state)$, is a closed set that bounds the robot  motion trajectory $\pos(t)$ starting at $t=0$ from an initial state $\state \in (\R^{\dimspace})^{\order}$ for all future times $t \geq 0$, i.e.,
\begin{align} \label{eq.motion_range_prediction_description}
\pos(t) \in \motionrange_{\ctrl_{\ctrlgoal}}(\state) \quad \forall t \geq 0,
\end{align} 
and is bounded by a Euclidean ball centered at $\ctrlgoal$  as
\begin{align}\label{eq.motion_range_bound_requirement}
\motionrange_{\ctrl_{\ctrlgoal}}(\state) \subseteq \ball(\ctrlgoal, \motionrangebound \norm{\state - \plist{ \ctrlgoal, \mat{0}, \ldots, \mat{0} }}) \quad \forall \state \in \plist{\R^{d}}^{\order},
\end{align} 
where $\motionrangebound \in \R_{> 0}$ is a fixed positive constant.
Note that the bound on motion range prediction in \refeq{eq.motion_range_bound_requirement} implies that the motion prediction asymptotically converges to a single point at $\ctrlgoal$ since the feedback motion control $\ctrl_{\ctrlgoal}$ is globally asymptotically stablizing the zero-motion position $\ctrlgoal$, i.e.,
\begin{align}
\lim_{t \rightarrow \infty} \motionrange_{\ctrl_{\ctrlgoal}}(\state) = \clist{\ctrlgoal}.
\end{align}
Hence, a bounded motion range prediction in \refeq{eq.motion_range_bound_requirement} ensures that the (collision) distance of the motion range prediction $\motionrange_{\ctrl_{\ctrlgoal}}(\state)$ to the free space boundary  $\partial \freespace$ might be zero only for a finite amount of time for any goal $\ctrlgoal \in \mathring{\freespace}$ in the free space interior $\mathring{\freespace}$, which is crucial to avoid undesired critical points, as later discussed in the proof of \refprop{prop.Stability}.

In general, determining collision-free states of high-order dynamical systems is a hard task in kinodynamic motion planning \cite{lavalle_kuffner_IJRR2001}. 
The availability of a motion range prediction allows one to effectively identify certain collision-free states of $n^\text{th}$-order dynamical systems because having motion range prediction in the free space implies safe robot motion, i.e.,
\begin{align}
\motionrange_{\ctrl_{\ctrlgoal}}(\state) \subseteq \freespace \Longrightarrow \pos(t) \in \freespace \quad \forall t \geq 0.
\end{align}
\begin{corollary} \label{cor.CollisionFreeState}
A robot state $\state \in \plist{\R^{\dimspace}}^{\order}$ is collision free under the control policy $\ctrl_{\ctrlgoal}$ if there exists a goal position $\ctrlgoal \in \freespace$ relative to which the associated motion range prediction is in the free space, i.e., $\motionrange_{\ctrl_{\ctrlgoal}} (\state) \subseteq \freespace$.
\end{corollary}
\noindent Accordingly, we shall continuously monitor the safety of the closed-loop robot motion with respect to its goal position by measuring the distance of the predicted motion range to the free space boundary as described below.

\subsubsection{Safety Assessment}

Given a motion range prediction $\motionrange_{\ctrl_{\ctrlgoal}}(\state)$ for the closed-loop robot dynamics in \refeq{eq.FeedbackMotionControl}, the  safety level of the robot's trajectory $\pos(t)$ starting from an initial state $\state = (\pos^{(0)}, \pos^{(1)}, \ldots, \pos^{(\order-1)}) \in \plist{\R^{\dimspace}}^{\order}$ towards a given goal $\ctrlgoal \in \R^{\dimspace}$ is defined  as the minimum distance between the predicted motion range and the free space boundary as
\begin{align}\label{eq.SafetyLevel}
\safelevel(\state, \ctrlgoal) &:= \safedist(\motionrange_{\ctrl_{\ctrlgoal}}(\state), \partial \freespace),  \\
&:= 
\left \{
\begin{array}{@{}c@{\,\,}l}
\min\limits_{\substack{\vect{a} \in \motionrange_{\ctrl_{\ctrlgoal}}(\state) \\  \vect{b} \in \partial\freespace}} \norm{\vect{a} - \vect{b}} & \text{, if } \pos^{(0)} \in \freespace,  
\\
0 &\text{, otherwise,}
\end{array}
\right .
\end{align} 
where $\partial \freespace$ denotes the boundary of the free space $\freespace$.
Here, a safety level of zero means unsafe motion; and the higher the safety level is the safer the motion. 
Note that we consider being exactly on the boundary of the free space to be unsafe although it is, by definition in \refeq{eq.FreeSpace}, free of collisions.

A critical requirement of the safety level measure for our stability analysis performed below is that $\safelevel(\state, \ctrlgoal)$ is a locally Lipschitz continuous function of the robot state $\state$ and the goal $\ctrlgoal$. 
For example,  $\safelevel(\state, \ctrlgoal)$ is locally Lipschitz continuous if the motion range prediction can be expressed as an affine transformation of some fixed sets (e.g., the unit ball/simplex) based on a smooth function of the robot state $\state$ and goal $\ctrlgoal$.

\begin{lemma} \label{lem.DistanceLipschitz}
(\emph{Lipschitz Continuity of Minimum Set Distance})
The minimum distance between two compact sets under an affine transformation is Lipschitz continuous wrt. the affine transformation parameters, i.e., for any $X \! \subset \! \R^{m}$ and $Y \!\subset \!\R^{n}$
{\small
\begin{align}
\absval{\safedist \plist{f_{\mat{A}, \vect{b}}(X), Y} \!-\!  \safedist \plist{f_{\mat{A}, \vect{b}'}(X), Y}} &  \leq \norm{\vect{b} - \vect{b}'}, \label{eq.DistanceLipschitz_b}
\\
\absval{\safedist \plist{f_{\mat{A}, \vect{b}}(X), Y} \!-\! \safedist \plist{f_{\mat{A}', \vect{b}}(X), Y}} &  \leq \norm{\mat{A} \!-\! \mat{A}'} \max_{\vect{x} \in X}\norm{\vect{x}}, \label{eq.DistanceLipschitz_A}
\end{align}
}%
where $\safedist(X,Y) = \min_{\vect{x} \in X, \vect{y} \in Y } \norm{\vect{x} - \vect{y}}$ is the minimum set distance, and  $f_{\mat{A},\vect{b}}(\vect{x}) \ldf \mat{A} \vect{x} + \vect{b}$ is an affine transformation parametrized by $\mat{A} \in \R^{n \times m}$ and $\vect{b} \in \R^{n}$.
\end{lemma}
\begin{proof}
See \refapp{app.DistanceLipschitz}.
\end{proof}

\subsubsection{Reference Governor}

A \textit{reference governor} is a first-order dynamical system with a position state $\govpos \in \R^{\dimspace}$ that follows a reference motion planner $\refplan_{\goal}: \refdomain \rightarrow \R^{\dimspace}$ towards a goal position $\goal \in \refdomain \subseteq \freespace$ as close as possible, based on the safety level $\safelevel(\state, \govpos) $ of the predicted robot motion  starting from state $\state \in \plist{\R^{d}}^\order$ towards the governor position $\govpos$.
We design the reference governor dynamics as follows:
\begin{subequations}\label{eq.ReferenceGovernor}
\begin{align} 
\govpos^{(1)} & = \govgain \proj_{\ball(\mat{0}, \safelevel(\state, \govpos))}(\refplan_{\goal}(\govpos)),  \label{eq.ReferenceGovernorKinematics}
\\ 
& = - \govgain \plist{\govpos - \proj_{\ball(\govpos, \safelevel(\state, \govpos))}(\govpos + \refplan_{\goal}(\govpos))}, 
\\
& = \govgain \min\plist{\safelevel(\state, \govpos), \norm{\refplan_{\goal}(\govpos)}}\frac{\refplan_{\goal}(\govpos)}{\norm{\refplan_{\goal}(\govpos)}}, \label{eq.ReferenceGovernorScaledDynamics}
\end{align}  
\end{subequations}
where $\govgain > 0$ is a fixed control gain for the governor, $\proj_{A}(\vect{b}) \ldf \argmin_{\vect{a} \in A} \norm{\vect{a} - \vect{b}}$ denotes the  metric projection of a point $\vect{b}$ onto a closed set $A$, and $\ball(\mat{0}, \sigma)$ is the Euclidean ball centered at the origin with radius $\sigma \geq 0$.
This design ensures that the governor is only allowed to move according to the reference planner $\refplan_{\goal}$ if the robot's motion relative to the governor is predicted to be safe, i.e., $\safelevel(\state, \govpos) = \safedist(\motionrange_{\ctrl_{\ctrlgoal}}(\state), \partial \freespace) > 0$.
Also note that the right-hand side of the reference governor dynamics in \refeq{eq.ReferenceGovernorScaledDynamics} is Lipschitz continuous since both the safety level $\safelevel(\state, \govpos)$ and the reference planner $\refplan_{\goal}(\govpos)$ are assumed to be Lipschitz.

\subsection{Working Principle of High Order Motion Planners}
\label{sec.GeneralWorkingPrinciple}

Our feedback motion planning framework consists of a low-level inner control loop and  a high-level outer  planning loop that bidirectionally interact with each other via their shared element --- reference governor, see \reffig{fig.general_framework}, to transfer the global navigation  properties of the first-order  reference motion planner to the $n^{\text{th}}$-order robot dynamics while ensuring  safety and stability.
The working principles of these control and planning loops and their interaction can be conceptually summarized as follows. 
\begin{itemize}
\item Low-Level Inner Control Loop: At the low level, the feedback motion control $\ctrl_{\govpos}(\state)$ constantly tries to stabilize the robot state $\state$ at the (potentially changing) governor position $\govpos$ with zero-motion, while the safety level $\safelevel(\state, \govpos)$ of the resulting robot motion is continuously monitored using the predicted robot motion range $\motionrange_{\ctrl_{\govpos}}(\state)$ with respect to the governor position $\govpos$.

\item High-Level Outer Planning Loop: At the high level, the reference governor tries to follow the reference motion planner $\refplan_{\goal}$ as close as possible, based on  the safety level  $\safelevel(\state, \govpos)$ of the current robot state $\state$ relative to the governor position $\govpos$, so that the governor asymptotically reaches to the global goal $\goal$ of the reference plan while ensuring safe robot motion. 

\item Low \& High Level Control and Planning Interaction: Hence, the reference governor defines a bidirectional interface between the inner control loop and the outer planning loop such that the reference motion plan is approximately transferred from the outer planning loop to the inner control loop while ensuring the safety and stability of the entire robot-governor system.   
\end{itemize}  

\noindent Finally, it is convenient to  write the overall dynamics of the robot-governor system using the closed-loop robot dynamics in \refeq{eq.FeedbackMotionControl} and the governor dynamics in \refeq{eq.ReferenceGovernor}  as
\begin{subequations}\label{eq.RobotGovernorDynamics}
\begin{align}
\pos^{(\order)} &= \ctrl_{\govpos}\plist{\state},  \label{eq.RobotDynamicsFinal}
\\
\govpos^{(1)} &=  \govgain \proj_{\ball(\mat{0}, \safelevel(\state, \govpos))}\plist{\refplan_{\goal}(\govpos)}, \label{eq.GovernorDynamicsFinal}
\end{align}
\end{subequations}
where $\refplan_{\goal}$ is a first-order reference dynamics and $\safelevel(\state, \govpos)$ is the robot's safely level relative to the governor as defined in \refeq{eq.SafetyLevel}. 
This clearly shows the strong coupling between the robot and the governor.

\subsection{Safety and Stability Properties}
\label{sec.GeneralCorrectness}

In this part, to prove the correctness of our motion planning framework, we first show that the robot-governor dynamics in \refeq{eq.RobotGovernorDynamics} result in a safe motion for both the $\order^{\text{th}}$-order robot and the first-order governor, and 
then show that they both asymptotically converge to zero-motion at the global goal $\goal$ by  following the first-order reference motion planner.    
\begin{proposition}\label{prop.Safety}
\emph{(Safety)}
Starting at $t = 0$ from any collision-free robot state $\state(0) \in \plist{\R^{\dimspace}}^{\order}$ relative to a collision-free governor position $\govpos(0) \!\in\! \refdomain \subseteq \freespace$ in the sense of \refcor{cor.CollisionFreeState}, under the robot-governor dynamics in \refeq{eq.RobotGovernorDynamics}, the robot's trajectory $\pos(t)$ and the governor trajectory $\govpos (t)$ stays collision free in $\freespace$  and $\refdomain$, \mbox{respectively, for all future times $t \geq 0$, i.e.,}
\begin{align}
\left.
\begin{array}{@{}c@{}}
\motionrange_{\ctrl_{\govpos}}\plist{\state(0)} \subseteq \freespace, 
\\
\govpos(0) \in \refdomain \subseteq \freespace
\end{array}
\right.
\,\,\, \Longrightarrow \,\,\,
\left .
\begin{array}{@{}c@{}}
\pos(t) \in \freespace,
\\
\govpos(t) \in \refdomain  \subseteq \freespace 
\end{array} 
\quad \forall t \geq 0.
\right .
\end{align}
\end{proposition}
\begin{proof}
The safety of the governor follows from the positive invariance of the collision-free domain $\refdomain$ of the reference planner $\refplanner_{\goal}$ because the governor dynamics in \refeq{eq.ReferenceGovernorScaledDynamics} equal to a non-negatively scaled version of the reference plan $\refplanner_{\goal}$ which ensures that $\refdomain$ is also positively invariant under  \refeq{eq.GovernorDynamicsFinal}. 

The safety of the robot relative to the governor (see \refcor{cor.CollisionFreeState}) can be observed in two steps. If the governor moves, i.e., $\govpos^{(1)}(t) \neq \mat{0}$ in \refeq{eq.GovernorDynamicsFinal}, then the positive safety level $\safelevel(\state(t),\govpos(t)) > 0$ implies  $\freespace\supset \motionrange_{\ctrl_{\govpos(t)}}\plist{\state(t)} \ni \pos(t)$.
Otherwise (i.e., if the governor is stationary), the safety of the robot trajectory is verified either at the start at $t = 0$ or just before the governor becomes stationary where we have $\safelevel(\state(t),\govpos(t)) = 0$, but $\motionrange_{\ctrl_{\govpos(t)}}\plist{\state(t)} \subseteq \freespace$. 
Hence, we have by definition \refeq{eq.motion_range_prediction_description} that $\pos(t) \in \motionrange_{\ctrl_{\govpos}}(\state(0))$ for all $t \geq 0$.  
\end{proof}
\noindent Although the motion range prediction is not necessarily in the free space for all times, it is useful to observe the motion prediction is inside the interior $\mathring{\freespace}$ of the free space when the governor moves since the safety level is strictly positive, i.e.,
{\small
\begin{align}\label{eq.PositiveSafetyLevel}
\govpos^{(1)}(t)  \neq \mat{0} \Longleftrightarrow \safelevel(\state(t), \govpos(t)) > 0 \Longleftrightarrow \motionrange_{\ctrl_{\govpos(t)}}(\state(t)) \subset \mathring{\freespace}. \!\!
\end{align}
}%

\begin{proposition} \label{prop.Stability}
\emph{(Stability)}
Starting from any collision-free robot state $\state(0) \in \plist{\R^{\dimspace}}^{\order}$ and collision-free governor position  $\govpos(0) \!\in\! \refdomain$ with positive safety level $\safelevel(\state(0), \govpos(0)) \!>\! 0$,  under the robot-governor dynamics in \refeq{eq.RobotGovernorDynamics}, both the robot and the governor asymptotically converge to zero motion at the global goal $\goal$ of the reference motion plan $\refplan_{\goal}$, i.e., 
\begin{align}
\lim_{t \rightarrow \infty} \pos^{(0)}(t) = \lim_{t \rightarrow \infty} \govpos^{(0)}(t) &= \goal, \\
\lim_{t \rightarrow \infty} \pos^{(i)}(t) = \lim_{t \rightarrow \infty} \govpos^{(1)}(t) & = \mat{0}, \quad \!\!\forall i = 1, \ldots, (\order-1). \!\!\!
\end{align}
\end{proposition}
\begin{proof}
The existence and uniqueness of the robot's and the governor's trajectory follows from the local Lipschitz continuity requirement of the feedback motion control $\ctrl_{\govpos}$, the reference motion plan $\refplan_{\goal}$, and the safety level $\safelevel(\state, \govpos)$ over the compact free space $\freespace$ \cite{khalil_NonlinearSystems2001}. 

It follows from \refeq{eq.PositiveSafetyLevel} that the initial  positive safety level $\safelevel(\state, \govpos) > 0$ ensures that the governor is inside the interior  $\mathring{\freespace}$ of the free space $\freespace$. 
Hence, since the reference motion planner (and its nonnegative scaling in \refeq{eq.GovernorDynamicsFinal}) is inward pointing  on the boundary of its domain $\refdomain \subseteq \freespace$, the governor stays inside the free space interior $\mathring{\freespace}$ for all future times. 

Since the reference planner $\refplan_{\goal}$ is a Lipschitz continuous vector field over its domain  $\refdomain$ with a unique stable point at $\goal$ whose domain of attraction includes $\refdomain$, by the converse Lyapunov theorem \cite{khalil_NonlinearSystems2001}, there exists a smooth Lyapunov function $U: \R^{\dimspace} \rightarrow \R$  for $\refplan_{\goal}$ such that $\nabla U( \govpos) \cdot \refplan_{\goal}(\govpos) < 0$ for all $ \govpos \in \refdomain \setminus \clist{\goal}$ \cite{arslan_koditschek_ICRA2017}.
Since the governor motion is determined by a nonnegatively scaled version of the reference planner as seen in \refeq{eq.ReferenceGovernorScaledDynamics}, the governor dynamics in \refeq{eq.GovernorDynamicsFinal} also satisfies  $\nabla U( \govpos) \cdot \govpos^{(1)} \leq 0$.
Therefore, since the closed-loop robot dynamics in \refeq{eq.RobotDynamicsFinal} is globally asymptotically stable at the governor position $\govpos$, which is contained in the free space interior $\mathring{\freespace}$ as discussed above, it follows from LaSalle's invariance principle  \cite{khalil_NonlinearSystems2001} using $U$ as a Lyapunov function candidate that both the robot and the governor asymptotically converge to zero motion at  the global goal $\goal$.    
\end{proof}

\noindent Hence, one can conclude  from Propositions \ref{prop.Safety} \& \ref{prop.Stability} that:
\begin{theorem} \emph{(Safe \& Stable Robot-Governor Navigation)}
Starting from any safe robot state $\state \in \plist{\R^{\dimspace}}^{\order}$ and governor position  $\govpos \in \refdomain$ with strictly positive safety level $\safelevel(\state, \govpos) > 0$,
the robot-governor dynamics in \refeq{eq.RobotGovernorDynamics} asymptotically brings the $\order^{\text{th}}$-order robot and the first-order governor to  the global goal $\goal$ according to the first-order reference motion plan $\refplan_{\goal}$ with no collisions along the way. 
\end{theorem}

\section{From Low to High Order Motion Planners:\\  Example Planning \& Control Elements} 
\label{sec.ExampleMotionControlPrediction}

In this section, we give example low-level control and high-level planning elements that can be used in our motion planning framework.
As low-level control elements, we provide an example choice of a standard feedback motion control policy for the $\order^{\mathrm{th}}$-order robot dynamics, and  present two example motion range prediction methods associated with that controller.
As high-level planning element, we briefly describe an existing first-order path pursuit reference planner for safe path tracking in cluttered environment \cite{arslan_koditschek_ICRA2017}.

\subsection{PhD Feedback Motion Control}

A classical control approach for stabilizing linear dynamical systems uses negative error feedback \cite{khalil_NonlinearSystems2001}.
Accordingly, we consider a simple negative error feedback policy  to bring the $\order^{\text{th}}$-order robot model to zero motion at any goal position.

\begin{definition} (\emph{PhD Motion Control})
For the $\order^{\text{th}}$-order robot dynamics, the \emph{proportional higher-order derivative (PhD) control} that globally asymptotically stabilizes the robot state $\state \in \plist{\R^{\dimspace}}^{\order}$ at any given goal (e.g., governor) position $\govpos \in \R^{\dimspace}$ at zero motion is defined as
\begin{align} \label{eq.PhD_control}
	\pos^{(\order)} = - \sum^{\order - 1}_{i = 0} \gain_i \pos^{(i)} + \gain_0 \govpos,
\end{align}
where the constant scalar control gains $\gain_0, \ldots, \gain_{n-1} \! \in \! \R$ ensure that the characteristic polynomial $p(\lambda)= \lambda^{\order} + \sum\limits_{i=0}^{\order - 1} \gain_i \lambda$ has roots with negative real parts.
\end{definition}

\noindent To leverage tools from linear system theory, it is convenient to represent the PhD motion control in \refeq{eq.PhD_control} as a first-order higher-dimensional dynamical system in the state space as
\begin{align} \label{eq.PhD_control_ss}
\dot{\state} = (\ctrlmat \otimes \mat{I}_{\dimspace \times \dimspace}) \plist{\state - \govstate},
\end{align}
where $\ctrlmat \! \in \! \R^{\order\times\order}$ is the (Hurtwiz) companion matrix associated with the control gains $\gain_0, \ldots, \gain_{n-1}$, and $\mat{I}_{\dimspace \times \dimspace}$ is the $\dimspace \times \dimspace$ identity matrix, and $\otimes$ denotes the Kronecker product. 
Here, $\state = (\pos^{(0)}, \pos^{(1)}, \ldots, \pos^{(\order -1)}) \in\R^{\order\dimspace}$ and  $\govstate = (\govpos, \mat{0}, \ldots, \mat{0})\in \R^{\order\dimspace}$, respectively, denote the robot's state and the zero-motion goal (e.g., governor) state.

In practice, one often avoids using underdamped PD control for second-order robotic systems to prevent oscillatory robot motion.
The notion of non-underdamped second-order systems can be extended to the PhD control of high-order dynamical systems as \emph{non-overshooting}.

\begin{definition}[\cite{arslan_isleyen_2022}] \label{def.NonovershootingPhDControl}
(\emph{Non-overshooting PhD Control})
A PhD feedback motion control of the form \refeq{eq.PhD_control} is said to be \emph{non-overshooting} if the control gains $\gain_0, \ldots, \gain_{\order-1}$  result in real negative characteristic polynomial roots, i.e., all eigenvalues of the associated companion matrix $\ctrlmat$ are real and negative.
\end{definition}

\subsection{Convex Motion Range Predictions for PhD Control}

We now present two motion range prediction methods  based on Lyapunov ellipsoids and Vandermonde simplexes that provide a convex bound on the robot's motion trajectory under the PhD motion control \cite{arslan_isleyen_2022}.

\subsubsection{Lyapunov Motion Range Prediction for PhD Control}
\label{sec.LyapunovMotionPrediction}

Using the state-space form of the closed-loop robot dynamics under PhD control in \refeq{eq.PhD_control_ss}, one can construct a quadratic Lyapunov function\footnote{Here, $\norm{\vect{x}}_{\mat{P}}$ denotes the weighted Euclidean norm associated with a positive definite matrix $\lyapmat \in \PDM^\order$.} relative to the goal position $\govpos$ as
\begin{align}
V_{\lyapmat} (\state) = \tr{(\state - \govstate)} \mat{P} (\state - \govstate) = \norm{\state - \govstate}_{\lyapmat}^2,
\end{align}
parametrized with a positive definite matrix $\lyapmat \in \PDM^\order$ that uniquely satisfies the Lyapunov equation
\begin{align} \label{eq.lyapunov_equation}
\tr{(\ctrlmat \otimes \mat{I}_{\dimspace \times \dimspace})} \lyapmat + \lyapmat (\ctrlmat \otimes \mat{I}_{\dimspace \times \dimspace}) + \tr{\decaymat} \decaymat = 0,
\end{align} 
for some matrix $\decaymat \in \R^{m \times \order}$ such that $(\ctrlmat \otimes \mat{I}_{\dimspace \times \dimspace}, \decaymat)$ is observable \cite{khalil_NonlinearSystems2001}. 
Since the time rate of change of the Lyapunov function satisfies
{\small
\begin{align}
\frac{\diff}{\diff t} V_{\lyapmat}(\state) & = \tr{(\state - \govstate)} (\tr{(\ctrlmat \otimes \mat{I}_{\dimspace \times \dimspace})} \lyapmat + \lyapmat (\ctrlmat \otimes \mat{I}_{\dimspace \times \dimspace})) (\state - \govstate), \nonumber \\
&  = -\norm{\decaymat (\state - \govstate)}^2 \leq 0,
\end{align}
}%
it follows from LaSalle's invariance principle \cite{khalil_NonlinearSystems2001} that the robot state $\state(t)$, starting from any initial state $\state(0) \in \R^{\order \dimspace}$ is contained in the Lyapunov ellipsoid $\elp(\govstate, \lyapmat^{-1}, \norm{\state(0) - \govstate}_{\lyapmat})$,  
\begin{align} \label{eq.lyapunov_state_bound}
\state(t) \in \elp(\govstate, \mat{P}^{-1}, \norm{\state(0) - \govstate}_{\lyapmat}) \quad \forall t \geq 0,
\end{align}
where $\elp(\elpctr, \elpmat, \elprad) : = \clist{\elpctr + \elprad \elpmat^{\frac{1}{2}} \vect{u} \big | \vect{u} \in \R^{n}, \norm{\vect{u}} \leq 1}$ is the ellipsoid centered at $\elpctr \in \R^n$ and associated with a positive semidefinite matrix $\elpmat \in \PSDM^n$ and a nonnegative scalar $\elprad \geq 0$, and
$\elpmat^{\frac{1}{2}}$ is a square root\footnote{One can compute the unique symmetric positive-definite square-root of $\elpmat$ as  $\elpmat^{\frac{1}{2}}~=~\mat{V}\diag\plist{\sqrt{\sigma_1}, \sqrt{\sigma_2}, \ldots, \sqrt{\sigma_\order}} \tr{\mat{V}}$ using the singular value decomposition $\elpmat = \mat{V} \diag\plist{\sigma_1, \sigma_2, \ldots, \sigma_n} \tr{\mat{V}} $, where $\diag\plist{\lambda_1, \lambda_2, \ldots, \lambda_n}$ is the diagonal matrix with  elements ${\lambda_1, \lambda_2, \ldots, \lambda_n}$.} of $\elpmat$ that satisfies $\elpmat^{\frac{1}{2}} \tr{(\elpmat^{\frac{1}{2}})\!}\! = \elpmat$.

To handle spatial safety constraints, one can bound the robot motion using an orthogonal projection of Lyapunov ellipsoids onto the position coordinates, as shown in \reffig{fig.motion_prediction_demo}. 
\begin{proposition} \label{prop.LyapunovEllipsoids}
(\emph{Lyapunov Ellipsoids for PhD Control})
Let $V_{\mat{P}}(\state) = \tr{(\state - \govstate)} \mat{P} (\state - \govstate) $ be a quadratic  Lyapunov  function, parameterized with a positive definite matrix $\mat{P} \in \PDM^\order$, for the stable closed-loop $\order^{\text{th}}$-order robot dynamics $\dot{\state} = (\ctrlmat \otimes \mat{I}_{\dimspace \times \dimspace}) (\state - \govstate)$  under the PhD control in \refeq{eq.PhD_control_ss}. 

The robot motion trajectory $\pos(t)$, starting at $t=0$ from any initial state $\state(0) \in \R^{\order \dimspace}$ towards a given goal  position $\govpos \in \R^{\dimspace}$ is bounded by the projected Lyapunov ellipsoid as
\begin{align}  \label{eq.ProjectedLyapunovEllipsoid}
	\pos(t) \in \elp(\govpos ,\tr{\mat{I}_{\order\dimspace\times\dimspace}}\lyapmat^{-1}\mat{I}_{\order\dimspace\times\dimspace},\norm{\state(0)\!-\! \govstate}_{\lyapmat} ) \quad \forall t \geq 0. \!\!\! 
\end{align} 
\end{proposition}
\begin{proof}
See \refapp{app.LyapunovEllipsoids}.
\end{proof}

\begin{figure}[t]
\centering
\begin{tabular}{@{}c@{\hspace{2mm}}c@{}}
\includegraphics[width = 0.45\columnwidth]{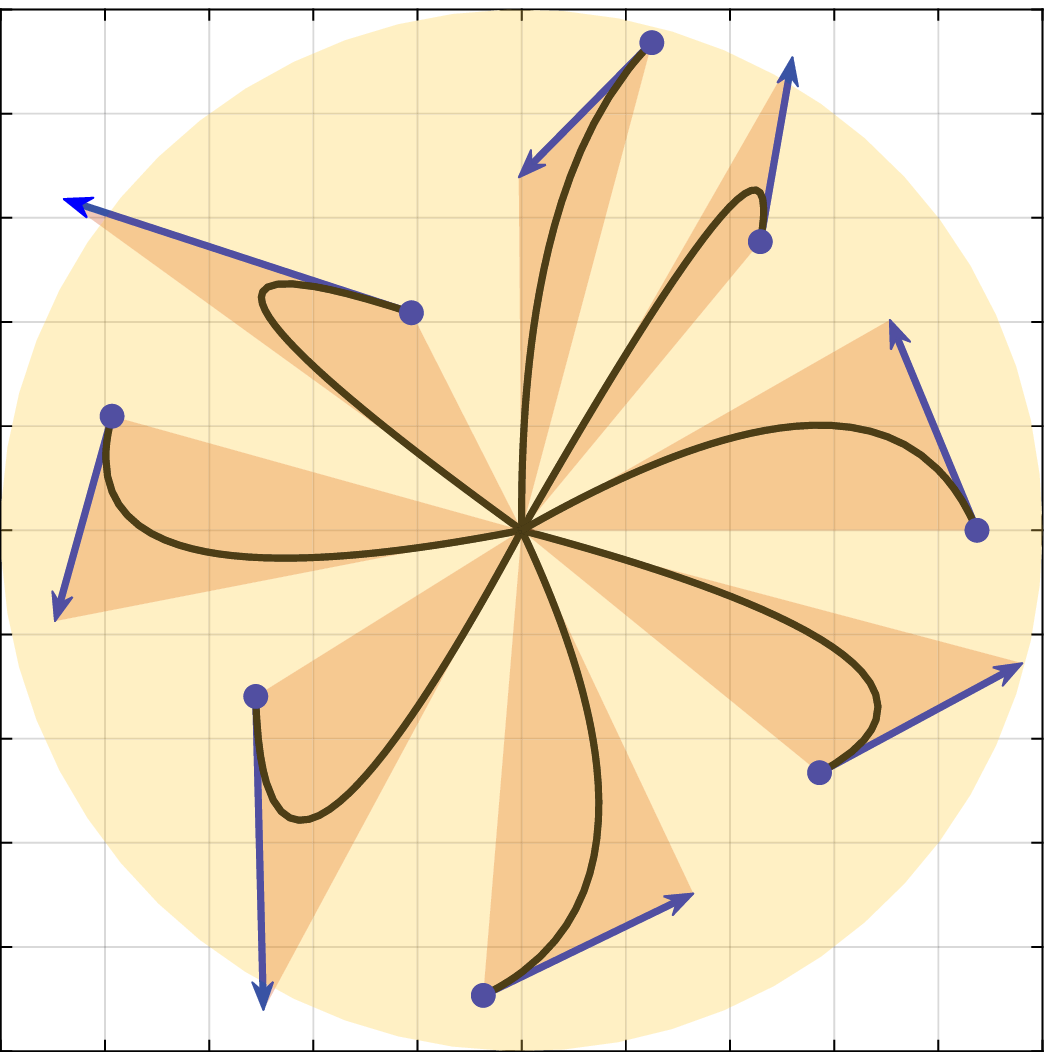} & 
\includegraphics[width = 0.45\columnwidth]{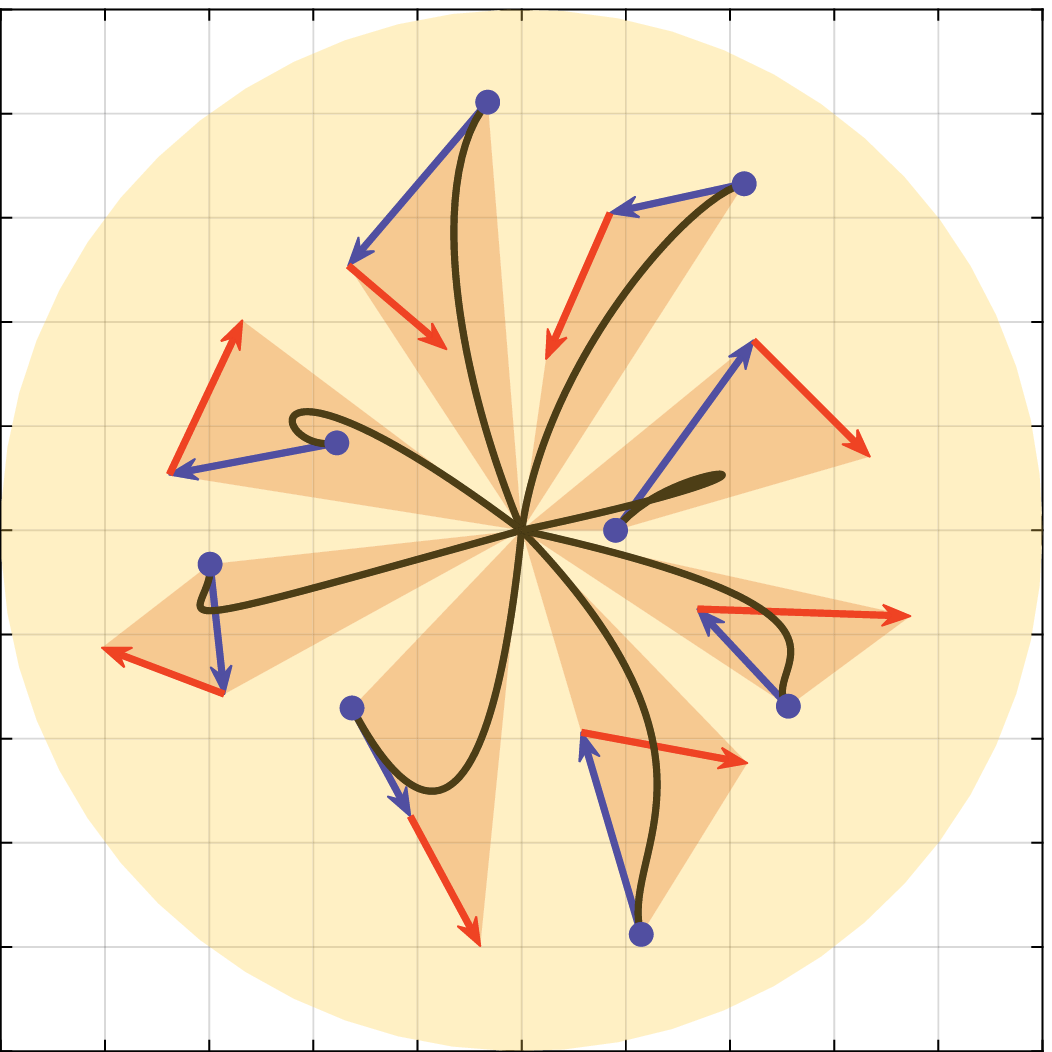}
\end{tabular}
\vspace{-1mm}
\caption{
Lyapunov (yellow circle) and Vandermonde (orange polygons) motion range prediction of PhD motion control towards the origin  for (left) second- and (right) third-order robot dynamics.
Here, the characteristic polynomial roots are uniformly spaced over $[-2, -1]$ and  $\decaymat=\mat{I}_{\order\dimspace \times \order\dimspace }$.}
\label{fig.motion_prediction_demo}
\vspace{-\baselineskip}
\end{figure}

As expected, the projected Lyapunov ellipsoids simply satisfy the boundedness requirement in \refeq{eq.motion_range_bound_requirement}. 

\begin{proposition} \label{prop.bounding_ball_lyapunov}
(Bounding Ball of Lyapunov Ellipsoids)
The projected Lyapunov ellipsoid $\elp(\govpos ,\tr{\mat{I}_{\order\dimspace\times\dimspace}}\lyapmat^{-1}\mat{I}_{\order\dimspace\times\dimspace},\norm{\state\!-\! \govstate}_{\lyapmat})$ in \refeq{eq.ProjectedLyapunovEllipsoid} is bounded by the Euclidean ball centered at $\govpos$ with radius $\motionrangebound \norm{\state - \govstate}$, i.e.,
\begin{align} \label{eq.bounding_ball_lyapunov}
	\!\!\!\! \elp(\govpos ,\tr{\mat{I}_{\order\dimspace\times\dimspace}}\lyapmat^{-\!1}\mat{I}_{\order\dimspace\times\dimspace}, \norm{\state\!-\!\govstate}_{\lyapmat} ) \!\subseteq\! \ball(\govpos, \motionrangebound \norm{\state \!-\! \govstate}),\!\!
\end{align}
where $\motionrangebound = \norm{\tr{\mat{I}_{\order\dimspace\times\dimspace}}\lyapmat^{-1}\mat{I}_{\order\dimspace\times\dimspace}}^{\frac{1}{2}} \norm{\lyapmat}^{\frac{1}{2}}$.
\end{proposition}
\begin{proof}
See \refapp{app.bounding_ball_lyapunov}.
\end{proof}

Projected Lyapunov ellipsoids also guarantee that the safety assessment in \refeq{eq.SafetyLevel} is Lipschitz continuous.

\begin{proposition}\label{prop.LyapunovSafetyLipschitz}
(\emph{Lipschitz Continuous Safety Assessment via Lyapunov Ellipsoids}) The projected Lyapunov ellipsoids induce a Lipshitz continuous safety level measure $\safelevel(\state, \govpos) = \safedist(\elp(\govpos ,\tr{\mat{I}_{\order\dimspace\times\dimspace}}\lyapmat^{-\!1}\mat{I}_{\order\dimspace\times\dimspace}, \norm{\state\!-\!\govstate}_{\lyapmat} ), \partial \freespace)$.
\end{proposition}
\begin{proof}
See \refapp{app.LyapunovSafetyLipschitz}.
\end{proof}

\begin{figure*}[t]
\centering
\begin{tabular}{@{}c@{\hspace{0.5mm}}c@{\hspace{0.5mm}}c@{\hspace{0.5mm}}c@{\hspace{0.5mm}}c@{}}
\includegraphics[width=0.199\textwidth]{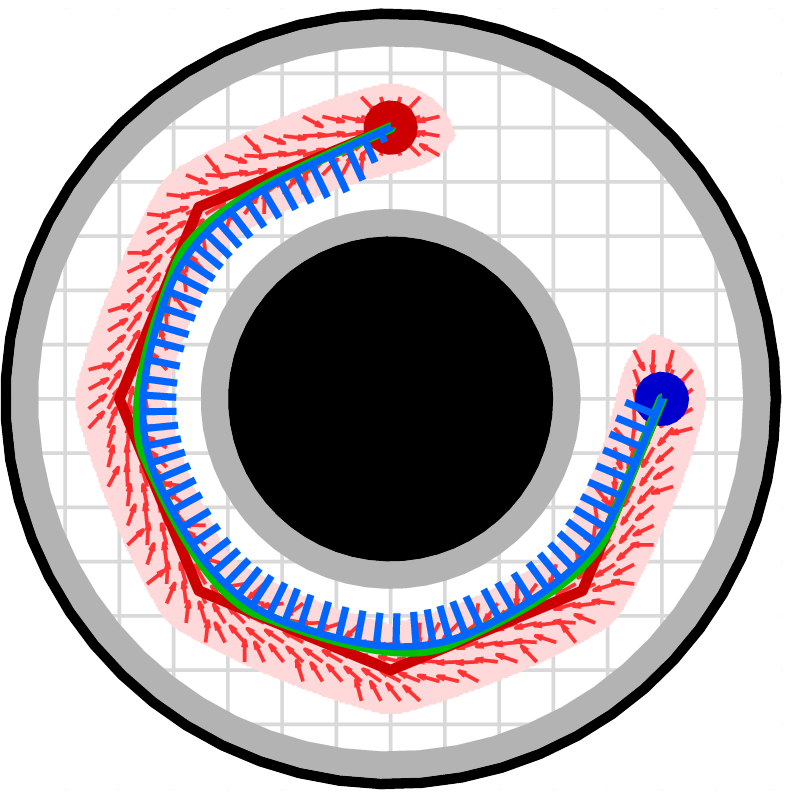} 
&
\includegraphics[width=0.199\textwidth]{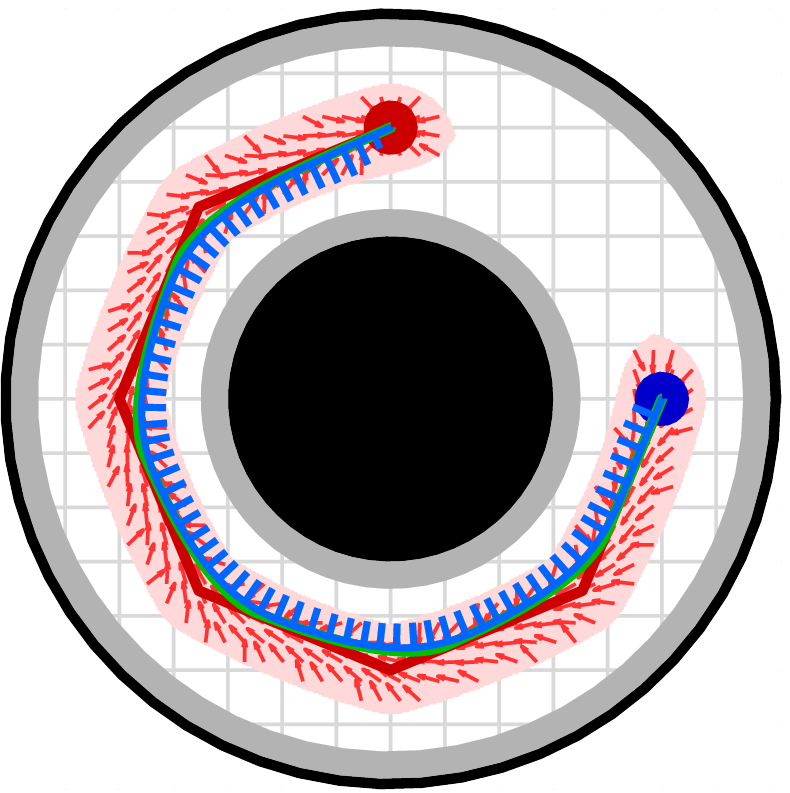} 
&
\includegraphics[width=0.199\textwidth]{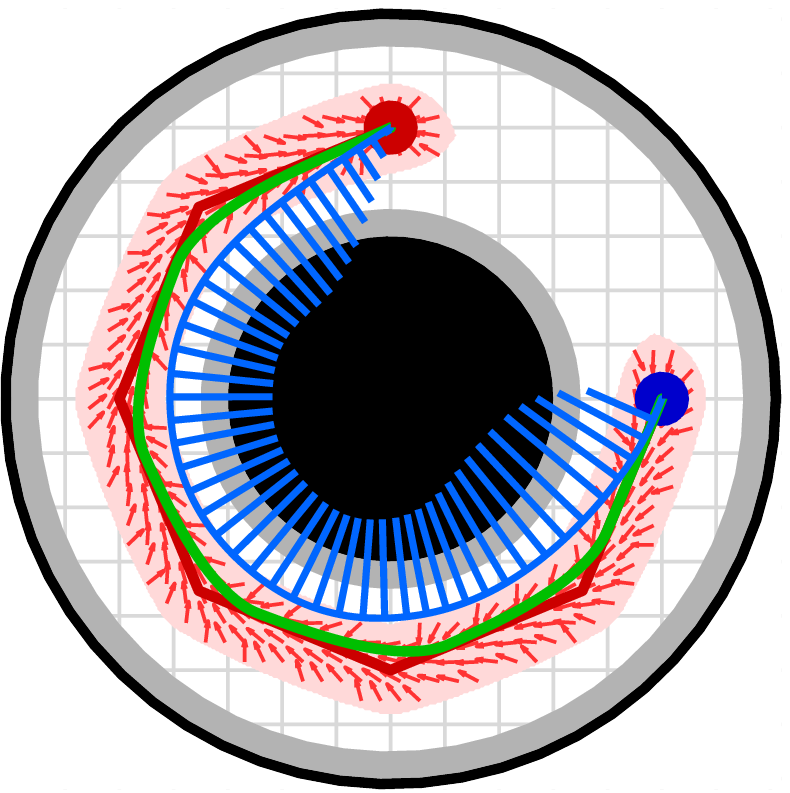} 
&
\includegraphics[width=0.199\textwidth]{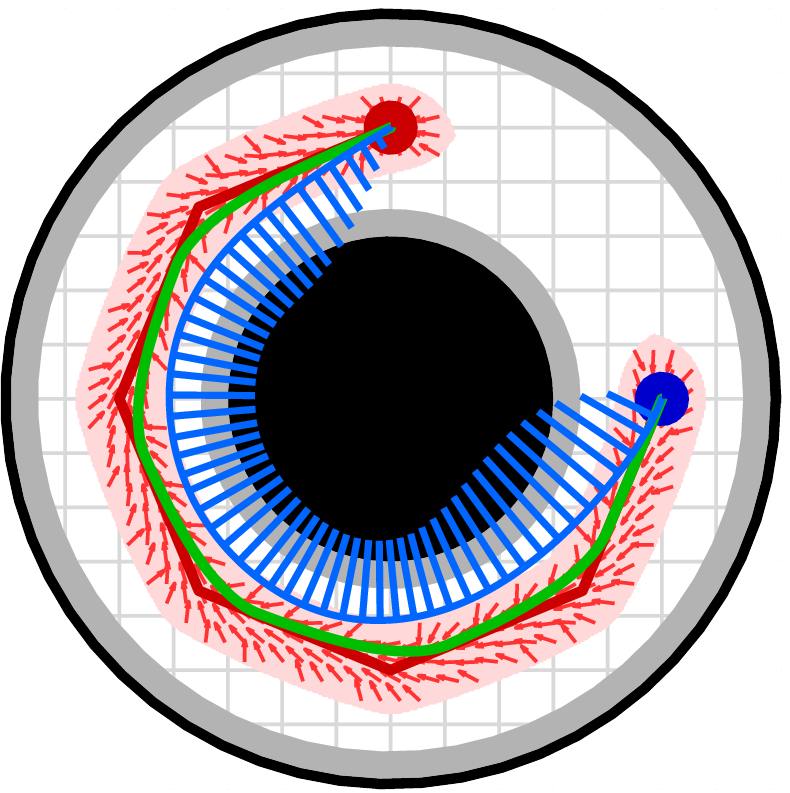}
&
\includegraphics[width=0.199\textwidth]{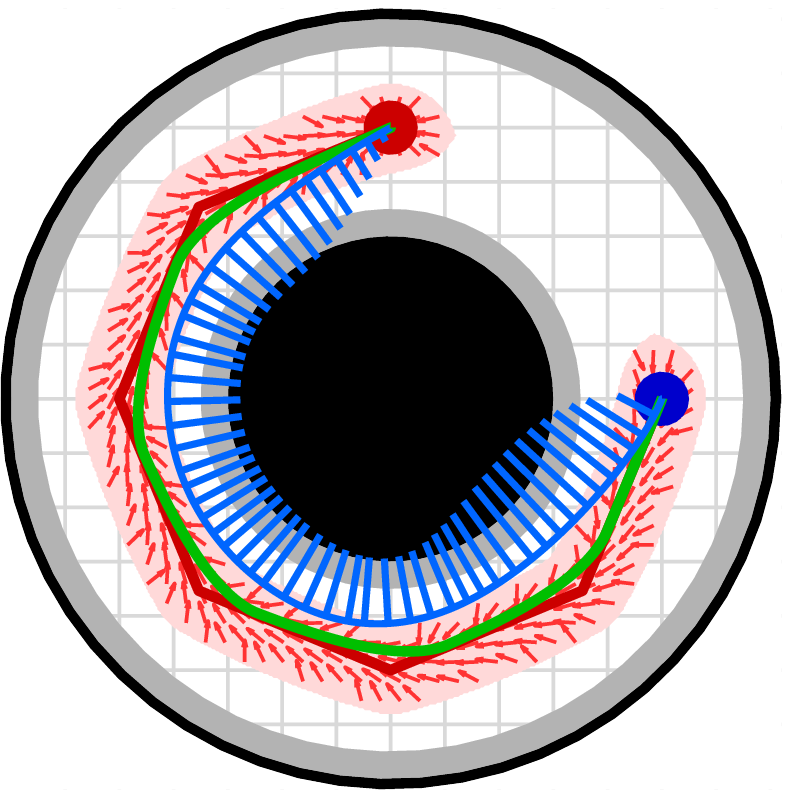}
\\[-1mm]
\small{(a)} & \small{(b)} & \small{(c)} & \small{(d)} & \small{(e)} 
\end{tabular}
\vspace{-2mm}
\caption{Safe and smooth robot navigation over a circular corridor using a first-order path pursuit reference planner.
The reference vector field (red arrows) are constructed based on a piecewise linear reference path (red line) starting at the blue circle and ending at the red circle. Workspace and configuration space obstacles are represented in black and gray, respectively.
Robot motion is illustrated by blue lines, where blue bars indicate robot speed, and governor motion is represented by green lines. Robot motion range is predicted using Lyapunov ellipsoids in (a, b) and Vandermonde simplexes in (c, d, e).  Robot is assumed to have the second-order dynamics in (a,c), the third-order dynamics in (b,d), and the forth-order dynamics in (e).}
\label{fig.SafeCorridorNavigation}
\vspace{-2mm}
\end{figure*}

\subsubsection{Vandermonde Motion Range Prediction}

As an alternative to Lyaponov ellipsoids, we recently introduce Vandermonde simplexes for bounding robot motion under non-overshooting PhD motion control (\refdef{def.NonovershootingPhDControl}).
As seen in \reffig{fig.motion_prediction_demo}, Vandermonde simplexes offer significantly more accurate and less conservative motion range prediction compared to Lyapunov ellipsoids because Vandermonde simplexes have a stronger dependency on control parameters and robot state.
\begin{proposition} [\cite{arslan_isleyen_2022}] \label{prop.motion_range_simplex}
(\emph{Vandermonde Simplexes for PhD Control})
Consider a non-overshooting PhD feedback motion controller for the $\order^{\text{th}}$-order robot model with control gains $\gain_0, \ldots, \gain_{\order-1} \in \R$ that are associated with real negative characteristic polynomial roots $\lambda_1, \ldots, \lambda_{\order} \in \R_{< 0}$. 

The robot trajectory $\pos(t)$,  starting at $t = 0$ from the initial state $\state(0) = (\pos^{(0)}_0,  \pos^{(1)}_0,  \ldots, \ \pos^{(\order-1)}_0)$ towards a given goal position $\govpos \in \R^\dimspace$, is contained in the Vandermonde simplex $\vsimplex_{\govpos}(\state(0))$ that is defined as
{\small
\begin{align} \label{eq.motion_range_simplex}
\pos(t) \!\in\! \vsimplex_{\govpos}(\state(0)) := \conv \plist{\! \govpos, \pos^{(0)}_0, \pos^{(0)}_0 \!+ \frac{\widehat{\gain}_1}{\widehat{\gain}_0} \pos^{(1)}_0, \ldots, \sum_{i=0}^{n-1} \frac{\widehat{\gain}_i}{\widehat{\gain}_0} \pos^{(i)}_0\!\!},\!\!
\end{align}
}%
where $\conv$ denotes the convex hull operator, and the positive coefficients $\widehat{\gain}_0, \ldots,  \widehat{\gain}_{\order-1}$ uniquely satisfy
\begin{align}
\!\!\prod_{\substack{\lambda_i \in \clist{\lambda_1, \ldots, \lambda_\order} \\ \lambda_i \neq \max(\lambda_1, \ldots, \lambda_\order)}} \!\!\!\! (\lambda - \lambda_i) 
= \sum_{i = 0}^{\order -1} \widehat{\gain}_i \lambda^{i}. \!\!   
\end{align}
\end{proposition}

An important property of Vandermonde simplexes is that $\vsimplex_{\govpos}(\state)$ is a linear transformation of the standard $\order$-simplex 
\mbox{$\simplex_{\order}:= \clist{(s_0, \ldots, s_\order) \!\in\! \R^{\order + 1} \big | \sum_{i = 0}^{\order} s_i = 1, s_i \geq 0 \,\,  \forall i   }$}, i.e.,
\begin{align}\label{eq.VandermondeSimplexLinearTransformation}
\vsimplex_{\govpos}(\state) = \clist{\mat{X} \vect{s} \, \big | \ ,\vect{s} \in \simplex_{\order}},
\end{align}
where the transformation matrix $\mat{X}$ is defined as a linear function of the robot state $\state$ and the goal position $\govpos$ as
\begin{align}
\mat{X} = \blist{\govpos, \pos^{(0)}, \ldots, \sum_{i=0}^{n-1} \frac{\widehat{\gain}_i}{\widehat{\gain}_0} \pos^{(i)} } \in \R^{\dimspace \times (\order +1)}.
\end{align}

As a motion range prediction method, Vandermonde simplexes satisfy the boundedness requirement in \refeq{eq.motion_range_bound_requirement}. 

\begin{proposition}\label{prop.VandermondeSimplexBound}
(Bounding Ball of Vandermonde Simplexes)
A Vandermonde simplex can be bounded by a Euclidean ball~as
\begin{align}
\vsimplex_{\govpos}(\state) \subseteq \ball(\govpos, \motionrangebound \norm{\state - \govstate}),
\end{align}
where $\motionrangebound =\sqrt{\order} \frac{\max(\widehat{\gain}_0, \ldots, \widehat{\gain}_{\order-1})}{\widehat{\gain}_0}$.
\end{proposition}
\begin{proof}
See \refapp{app.VandermondeSimplexBound}.
\end{proof}

Vandermonde simplexes also yield a Lipschitz continuous safety assessment as described in \refeq{eq.SafetyLevel}.  

\begin{proposition}\label{prop.VandermondeSafetyLipschitz}
(\emph{Lipschitz Continuous Safety Assessment via Vandermonde Simplexes})
The Vandermonde-simplex-based safety level measure $\safelevel(\state, \govpos) = \safedist(\vsimplex_{\govpos}(\state), \partial \freespace)$  is  Lipschitz continuous. 
\end{proposition}
\begin{proof}
See \refapp{app.VandermondeSafetyLipschitz}.
\end{proof}

\subsection{Path Pursuit Reference Planner}
\label{sec.PathPursuitReferencePlanner}

As a reference motion planner, we consider the ``move-to-projected-path-goal'' navigation policy in \cite{arslan_koditschek_ICRA2017} that constructs a first-order vector field around a given navigation path based on a safe pure pursuit path tracking approach \cite{coulter_TechReport1992}.

Let $\path : [0,1] \rightarrow \mathring{\freespace}$ be a continuous navigation path inside the free space interior $\mathring{\freespace}$, either generated by a standard path planner \cite{lavalle_PlanningAlgorithms2006} or determined by the user, that connects the start point $\path(0)$ to the end point  $\path(1)=\goal$. 
Accordingly, the first-order ``move-to-projected-path-goal'' law (a.k.a. path pursuit reference planner) $\refplan_{\path}: \refdomain_{\path} \rightarrow \R^{\dimspace}$ is defined  over its positively invariant domain $\refdomain_{\path}$ \cite{arslan_koditschek_ICRA2017}, 
\begin{align}
\refdomain_{\path} \ldf \clist{ \vect{q} \in \freespace \big | \safedist(\vect{q} , \path) \leq \safedist(\vect{q} , \partial \freespace)  },
\end{align}
as 
\begin{align}\label{eq.pathpursuitplanner}
\govpos^{(0)} = \refplan_{\path}(\govpos) = - \gain_{\path}(\govpos - \path^{*}(\govpos)),
\end{align}
where $\gain_{\path} > 0$ is a constant gain and the ``projected path goal'', denoted by $\path^*(\govpos)$,  is determined as
\begin{align}
\!\!\path^*(\govpos) \! \ldf\! \path \plist{\! 
\max \plist{\! \Big. \clist{  
\alpha \!\in\! [0,\!1]  \big|  \path(\alpha) \!\in\! \ball(\govpos,\!\safedist(\govpos,\partial\freespace)) 
\!} 
\!} \!\!}.\!\!\!
\end{align}
By construction, for piecewise continuously differentiable navigation paths, the path pursuit planner $\refplan_{\path}$ in \refeq{eq.pathpursuitplanner} is locally Lipschitz continuous and inward pointing on its domain boundary $\partial\refdomain_{\path}$, and it is asymptotically stable at $\path(1) = \goal$ whose domain of attraction includes the domain $\refdomain_{\path}$ \cite{arslan_koditschek_ICRA2017}.

\section{From Low to High Order Motion Planners:\\ Numerical Simulations}
\label{sec.NumericalSimulations}

In this section, we provide numerical simulations\footnote{For all simulations, we set the path pursuit planner gain $\gain_\path = 1$, the governor gain $\govgain = 4$ and the characteristic polynomial roots of the PhD control are uniformly spaced over $[-2, -1]$. All simulations are obtained  by numerically solving the associated robot-governor dynamics using the \texttt{ode45} function of MATLAB. Please see the accompanying video for the animated robot-governor motion.} to demonstrate smooth extensions of the first-order path pursuit planner for safe navigation of high-order robots around obstacles, using the PhD motion control and safety assessment based on Lyapunov ellipsoids and Vandermonde simplexes. 
In particular, we systematically investigate the role of motion range prediction and system order on governed robot motion.

\begin{figure*}[t]
\centering
\begin{tabular}{@{}c@{\hspace{0.01mm}}c@{\hspace{0.01mm}}c@{}}
\includegraphics[width=0.335\textwidth]{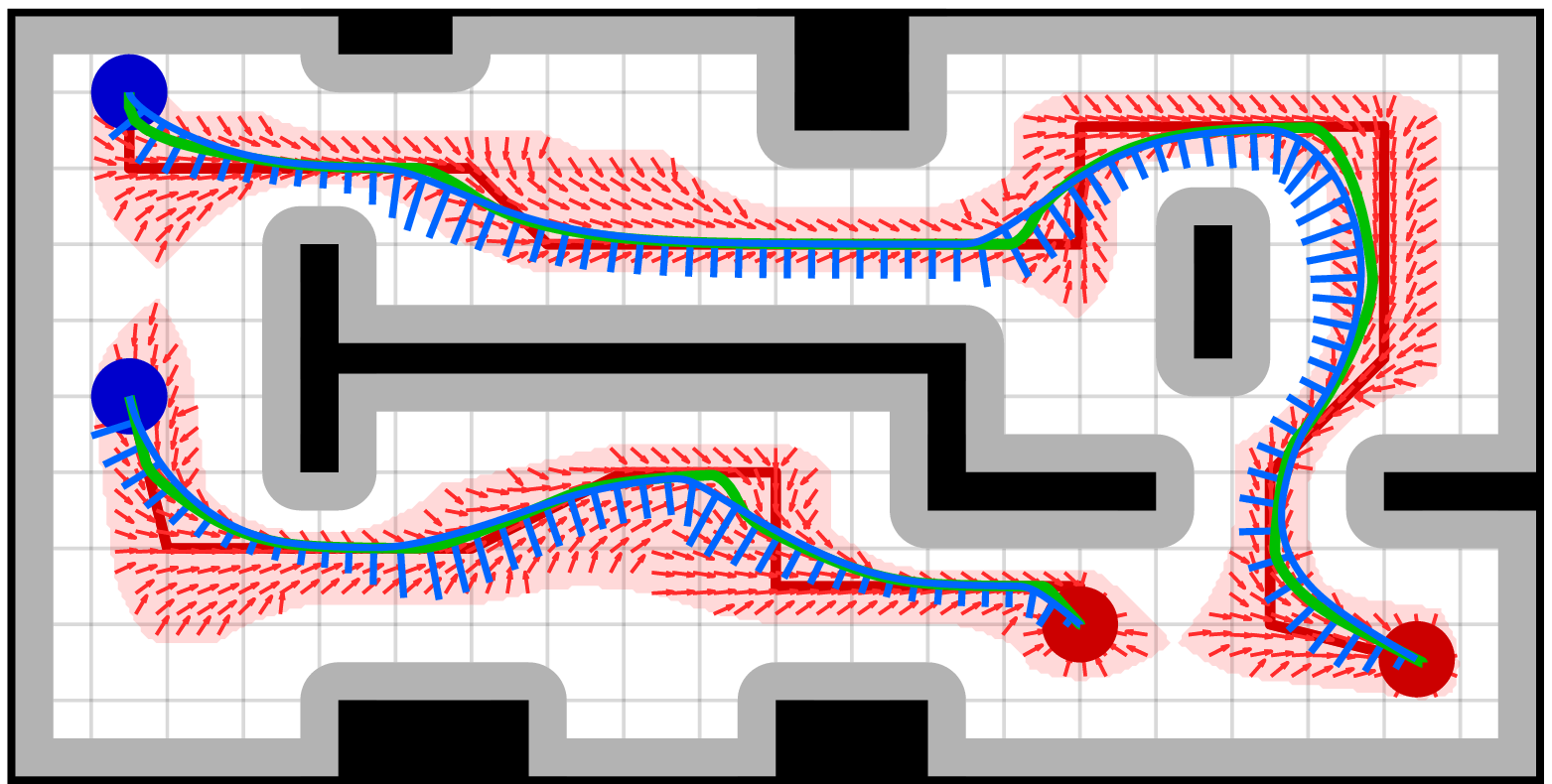} 
&
\includegraphics[width=0.335\textwidth]{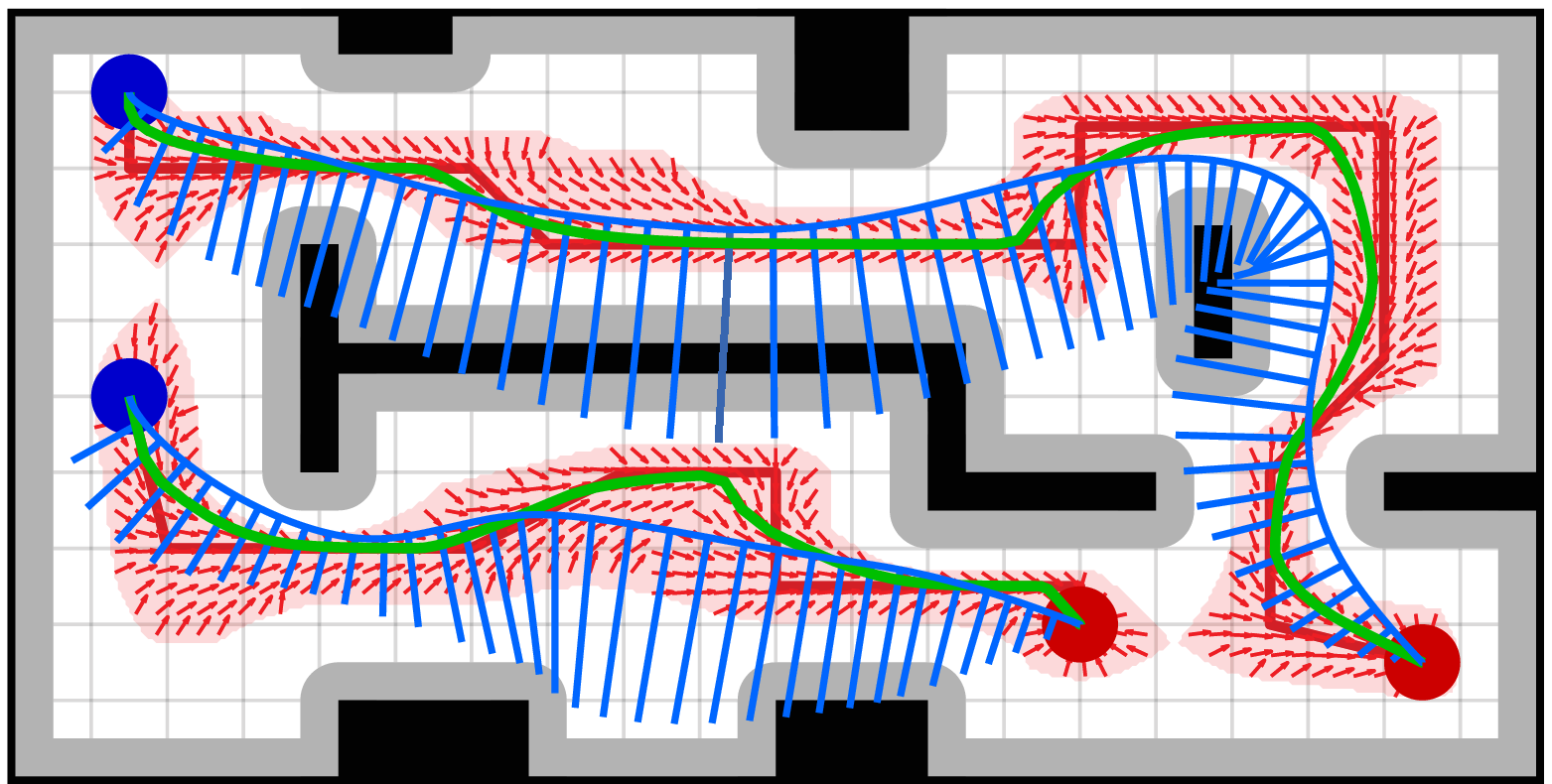} 
&
\includegraphics[width=0.335\textwidth]{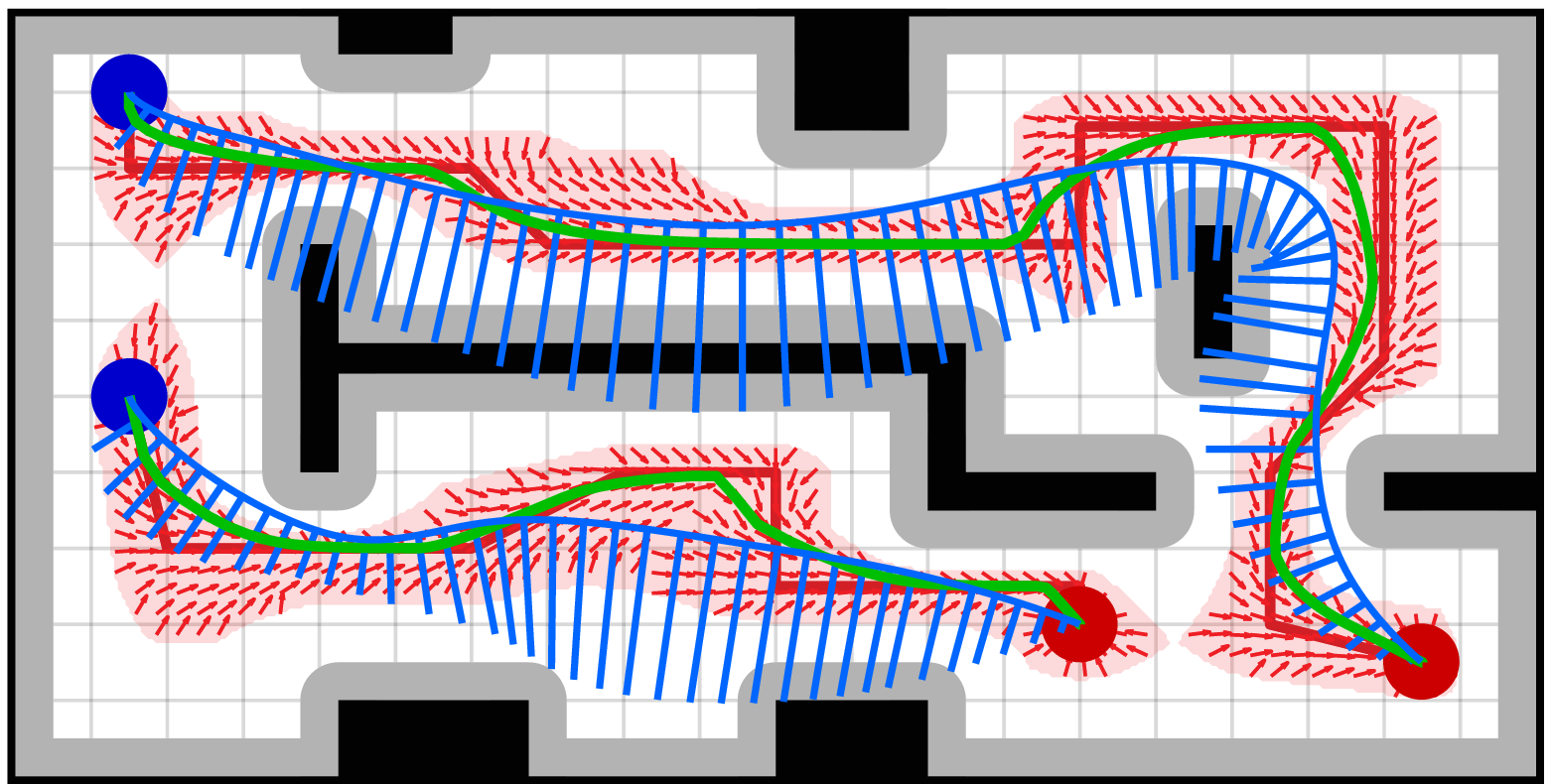} 
\end{tabular}
\vspace{-1mm}
\caption{Safe and smooth robot navigation in a cluttered environment using a first-order path pursuit reference planner (red arrows) that is  constructed based on a piecewise linear reference path (red line). 
Robot motion (blue lines)  for second-order (left, middle) and third-order (right) dynamics is constantly predicted relative to governor motion (green lines) using (left) Lyapunov ellipsoids and (middle, right) Vandermonde simplexes, where the robot speed is indicated by blue bars.   
}
\label{fig.cluttered_environment}
\vspace{-2mm}
\end{figure*}

\subsection{Safe Navigation in a Circular Corridor}

As a first example, we consider safe navigation in a circular narrow corridor since smooth motion planning and control in such tight spaces is a well-known challenge for high-order robots \cite{li_ICRA2020}. 
In \reffig{fig.SafeCorridorNavigation}, we illustrate the resulting robot trajectories and speeds for the second-, third-, and fourth-order robot dynamics, where the safety of robot motion relative to the governor is monitored using Lyapunov ellipsoids and Vandermonde simplexes.
As expected, the robot can reach to the desired destination of the path pursuit reference planner irrespective of the order of robot dynamics and motion range prediction, but the resulting robot motion significantly differs in terms of robot speed and so travel time.
As seen in \reffig{fig.motion_prediction_demo}, Lyapunov ellipsoids are more conservative in estimating robot motion range because Vandermonde simplexes has stronger dependency on robot's state and control parameters (see \refprop{prop.LyapunovEllipsoids} and \refprop{prop.motion_range_simplex}). 
As a result, Vandermonde simplexes always yield faster robot motion.
We observe that the robot is more cautious about sideways collisions with corridor walls when its motion range is predicted by Lyapunov ellipsoids.
Also, since motion range prediction is naturally more conservative for higher-order robot dynamics, we see that robot motion gets slower with increasing system order.

\subsection{Safe Navigation in a Cluttered Environment}

To demonstrate how motion prediction plays a critical role in adapting robot motion around complex obstacles, we consider safe robot navigation in a cluttered environment, illustrated in \reffig{fig.cluttered_environment}.
In such a environment, one might naturally expect that the robot slows down while making a turn around obstacles and speeds up if there is a large opening in front of the robot. 
This, of course, depends on many design parameters in planning and control. We observe in our numerical studies that motion prediction is a key factor. 
Conservative motion prediction like Lyapunov ellipsoids often has a tendency to slow down robot motion because the predicted robot motion cannot be accurately related to the environment. 
As seen in \reffig{fig.cluttered_environment} Lyapunov motion range prediction is limited in adapting robot motion around obstacles, whereas
Vandermonde simplexes  allow  the robot to leverage available space for faster navigation without compromising safety since Vandermonde simplxes are able to capture robot motion more accurately.
Therefore, accurate motion prediction is crucial for generating adaptive agile highly dynamic robot motion around complex (potentially dynamic) obstacles.
Finally, it is useful to note that  faster robot trajectories significantly deviate from the reference (governor) trajectories because a reference plan in our motion planning framework is considered as a high-level flexible navigation plan  towards the destination point as opposed to precise trajectory tracking control \cite{quinlan_khatib_ICRA1993}.

\section{Conclusions}
\label{sec.Conclusions}

In this paper, we introduce a provably correct feedback motion planning framework that extends safe navigation properties of simple first-order reference motion planners to  high-order robot dynamics using reference governors and safety assessment of predicted robot motion.
Our approach separates high-level planning and low-level control challenges to address them individually using standard tools from the motion planning and control literature.
We  establish a novel bidirectional interface between planning and control using reference governors and motion prediction. 
In addition to specifying generic motion planning elements, we provide example constructions  for motion control, prediction, and reference planning.    
We conclude that accurate motion prediction plays a key role in closing the gap between high-level planning and low-level control for generating agile robot motion.
In numerical simulations, we demonstrate the effectiveness of the proposed motion planning framework.

Work now in progress focuses on extending our motion planning framework to nonholonomically constrained robots such as differential drive vehicles and drones by designing new nonholonomic motion prediction algorithms. 
Another promising research direction is sensor-based safe robot navigation in dynamic and unknown environments \cite{arslan_koditschek_IJRR2019}.


\bibliographystyle{IEEEtran}
\bibliography{references}

%
\appendices 

\section{Proofs}

\subsection{Proof of \reflem{lem.DistanceLipschitz}}
\label{app.DistanceLipschitz}

\begin{proof}
The Lipschitz continuity of $\safedist \plist{f_{\mat{A}, \vect{b}}(X), Y}$ w.r.t. parameter $\vect{b}$ can be shown using the triangle inequality as
{\small
\begin{align}
\safedist \plist{f_{\mat{A}, \vect{b}}(X), Y} &  = \min_{\substack{\vect{x} \in X \\ \vect{y} \in Y}} \norm{\mat{A}\vect{x} + \vect{b} + \vect{b}'- \vect{b}' - \vect{y}}, 
\\
& \leq \min_{\substack{\vect{x} \in X \\ \vect{y} \in Y}} \norm{\mat{A}\vect{x} + \vect{b}' - \vect{y}} + \norm{\vect{b}- \vect{b}'}, 
\\
& = \safedist \plist{f_{\mat{A}, \vect{b}'}(X), Y}  + \norm{\vect{b}- \vect{b}'},
\end{align}
}%
which, by symmetry, implies \refeq{eq.DistanceLipschitz_b}.

Similarly, the Lipschitz continuity of $\safedist \plist{f_{\mat{A}, \vect{b}}(X), Y}$ w.r.t. parameter $\mat{A}$ can be shown using the triangle inequality and the submultiplicative property of matrix norms as 
{\small
\begin{align}
\safedist \plist{f_{\mat{A}, \vect{b}}(X), Y} \hspace{-15mm}& \hspace{+15mm} = \min_{\substack{\vect{x} \in X \\ \vect{y} \in Y}} \norm{\mat{A} \vect{x} + \mat{A}'\vect{x} - \mat{A}'\vect{x} + \vect{b} - \vect{y}},  
\\
& \leq \min_{\substack{\vect{x} \in X \\ \vect{y} \in Y}} \norm{\mat{A}' \vect{x} + \vect{b} - \vect{y}} + \norm{(\mat{A} - \mat{A'}) \vect{x}},  
\\
& \leq \min_{\substack{\vect{x} \in X \\ \vect{y} \in Y}} \norm{\mat{A}' \vect{x} + \vect{b} - \vect{y}} + \norm{\mat{A} - \mat{A'}} \norm{\vect{x}}, 
\\
& \leq \min_{\substack{\vect{x} \in X \\ \vect{y} \in Y}} \norm{\mat{A}' \vect{x} + \vect{b} - \vect{y}} + \norm{\mat{A} - \mat{A}'} \max_{\vect{x}' \in X}\norm{\vect{x}'}, 
\\
& = \safedist \plist{f_{\mat{A}', \vect{b}}(X), Y} + \norm{\mat{A} - \mat{A}'} \max_{\vect{x} \in X}\norm{\vect{x}},
\end{align}
}%
which, by symmetry, implies \refeq{eq.DistanceLipschitz_A}.  
\end{proof}

\subsection{Proof of \refprop{prop.LyapunovEllipsoids}}
\label{app.LyapunovEllipsoids}

\begin{proof}
As discussed above in \refeq{eq.lyapunov_state_bound}, the robot state trajectory $\state(t)$ is contained in the Lyapunov ellipsoid $\elp(\govstate, \mat{P}^{-1}, \norm{\state(0) - \govstate}_{\lyapmat})$ for all $t \geq 0$.
Accordingly, the robot motion trajectory $\pos(t)$ can be bounded by taking the orthogonal projection of the state space onto the positional subspace via the linear transformation $\state \mapsto \tr{\mat{I}_{\order \dimspace \times \dimspace}} \state$ as \cite{arslan_isleyen_2022}
\begin{align}
\pos(t) &= \tr{\mat{I}_{\order \dimspace \times \dimspace}} \state(t) \\
&\in  \clist{\tr{\mat{I}_{\order \dimspace \times \dimspace}}\mat{z} | \mat{z} \in \elp(\govstate, \lyapmat^{-1}, \norm{\state(0) - \govstate}_{\lyapmat}) },
\\
& = \elp(\govpos ,\tr{\mat{I}_{\order\dimspace\times\dimspace}}\Lyap^{-1}\mat{I}_{\order\dimspace\times\dimspace},\norm{\state(0)\!-\! \govstate}_{\lyapmat} ),
\end{align}
which completes the proof.
\end{proof}

\subsection{Proof of \refprop{prop.bounding_ball_lyapunov}}
\label{app.bounding_ball_lyapunov}

\begin{proof}
Let $\vect{z} = \govpos +  \norm{\state\!-\!\govstate}_{\lyapmat} (\tr{\mat{I}_{\order\dimspace\times\dimspace}}\lyapmat^{-\!1}\mat{I}_{\order\dimspace\times\dimspace} )^{\frac{1}{2}} \vect{u} $ be any point in $\elp(\ctrlgoal ,\tr{\mat{I}_{\order\dimspace\times\dimspace}}\lyapmat^{-\!1}\mat{I}_{\order\dimspace\times\dimspace}, \norm{\state\!-\!\govstate}_{\lyapmat} ) $, where $\vect{u} \in \R^{\dimspace}$ with $\norm{\vect{u}} \leq 1$.
One can verify the result using $\norm{\state\!-\!\govstate}_{\lyapmat} = \norm{\lyapmat^{\frac{1}{2}} (\state - \govstate)}$ and the sub-multiplicativity property of matrix norms as
\begin{align}
\norm{\vect{z} - \govpos} & = \norm{\lyapmat^{\frac{1}{2}}(\state\!-\!\govstate)} \norm {(\tr{\mat{I}_{\order\dimspace\times\dimspace}}\lyapmat^{-\!1}\mat{I}_{\order\dimspace\times\dimspace} )^{\frac{1}{2}} \vect{u}}, 
\\
& \leq  \norm{\lyapmat^{\frac{1}{2}}} \norm{\state\!-\!\govstate} \norm {(\tr{\mat{I}_{\order\dimspace\times\dimspace}}\lyapmat^{-\!1}\mat{I}_{\order\dimspace\times\dimspace} )^{\frac{1}{2}}}  \norm{\vect{u}},
\\
& \leq \norm{\lyapmat^{\frac{1}{2}}}  \norm {(\tr{\mat{I}_{\order\dimspace\times\dimspace}}\lyapmat^{-\!1}\mat{I}_{\order\dimspace\times\dimspace} )^{\frac{1}{2}}} \norm{\state\!-\!\govstate},
\end{align}  
which completes the proof.
\end{proof}

\subsection{Proof of \ref{prop.LyapunovSafetyLipschitz}}
\label{app.LyapunovSafetyLipschitz}

\begin{proof}
Observe that the projected Lyapunov ellipsoid is an affine transformation of the $\dimspace$-dimensional unit ball $\ball(\mat{0}_\dimspace, 1)$, 
\begin{align}
\elp(\govpos ,\tr{\mat{I}_{\order\dimspace\times\dimspace}}\lyapmat^{-\!1}\mat{I}_{\order\dimspace\times\dimspace}, \norm{\state\!-\!\govstate}_{\lyapmat} ) = f_{\mat{A}, \vect{b}}(\ball(\mat{0}_\dimspace, 1)),
\end{align}
where $f_{\mat{A}, \mat{b}}(\vect{z}) =  \mat{A} \vect{z} + \vect{b}$ with transformation parameters $\mat{A} =  \norm{\state\!-\!\govstate}_{\lyapmat} (\tr{\mat{I}_{\order\dimspace\times\dimspace}}\lyapmat^{-\!1}\mat{I}_{\order\dimspace\times\dimspace})^{\frac{1}{2}} $ and $\vect{b} = \govpos$ that are smooth Lipschitz functions of  the robot state $\state$ and the goal $\govpos$.
Therefore, the result follows from \reflem{lem.DistanceLipschitz}.
\end{proof}

\subsection{Proof of \refprop{prop.VandermondeSimplexBound}}
\label{app.VandermondeSimplexBound}

\begin{proof}
Any point  $\vect{z}  \in \vsimplex_{\govpos}(\state)$ can be written as a convex combination of Vandermonde simplex points based on some $(s_0, \ldots, s_\order) \in \simplex_{\order}$ as
{\small
\begin{align}
\vect{z} &= s_0 \govpos + \sum_{i=1}^{\order} s_i \sum_{j=0}^{i-1} \frac{\widehat{\gain}_j}{\widehat{\gain}_0} \pos^{(j)}, 
\\ 
&= \govpos + \sum_{i=1}^{\order} s_i (\pos^{(0)} - \govpos) + \sum_{i=1}^{\order} s_i \sum_{j=1}^{i-1} \frac{\widehat{\gain}_j}{\widehat{\gain}_0} \pos^{(j)}. 
\end{align}
}%
Hence, since $\widehat{\gain}_j > 0$, $s_i\geq 0 $ and $\sum_{i=0}^{\order} s_i = 1$, we have
{\small
\begin{align}
\norm{\vect{z} - \vect{y}} &= \norm{\sum_{i=1}^{\order}s_i (\pos^{(0)} - \govpos) + \sum_{i=1}^{\order} s_i \sum_{j=1}^{i-1}  \frac{\widehat{\gain}_j}{\widehat{\gain}_0} \pos^{(j)}},
\\
& \leq \sum_{i=1}^{\order}s_i \norm{\pos^{(0)} - \govpos} +  \sum_{i=1}^{\order} s_i \sum_{j=1}^{i-1}  \frac{\widehat{\gain}_j}{\widehat{\gain}_0} \norm{\pos^{(j)}} ,
\\
& \leq \norm{\pos^{(0)} - \govpos} + \sum_{j=1}^{\order-1}  \frac{\widehat{\gain}_j}{\widehat{\gain}_0} \norm{\pos^{(j)}} ,
\\
& \leq \frac{\max(\widehat{\gain}_0, \ldots, \widehat{\gain}_{\order-1})}{\widehat{\gain}_0} \bigg ( \!\norm{\pos^{(0)}\! - \govpos} + \sum\limits_{j=1}^{\order-1} \norm{\pos^{(j)}}\! \bigg),
\\
& \leq  \frac{\max(\widehat{\gain}_0, \ldots, \widehat{\gain}_{\order-1})}{\widehat{\gain}_0} \sqrt{\order} \norm{\state - \govstate}
\end{align}
}%
where the last inequality follows from the norm equivalence $\norm{\vect{a}}_1 \leq \sqrt{\order} \norm{\vect{a}}_2$ for any $\vect{a} \in \R^{\order}$. 
Thus, the result follows.
\end{proof}

\subsection{Proof of \refprop{prop.VandermondeSafetyLipschitz}}
\label{app.VandermondeSafetyLipschitz}

\begin{proof}
The result directly follows from \reflem{lem.DistanceLipschitz} since the Vandermonde simplex $\vsimplex_{\govpos}(\state)$ is a linear transformation of the standard $\order$-simplex $\simplex_{\order}$,  as expressed in \refeq{eq.VandermondeSimplexLinearTransformation}. 
\end{proof}


\end{document}

%% file: packages.tex
\usepackage{xcolor}
\usepackage{comment}

\usepackage{cite}

\usepackage[mathcal]{euscript}

\usepackage[cmex10]{amsmath}
\usepackage{amssymb}

\usepackage{amsthm} 
  
\usepackage{dsfont} 

\usepackage{epsfig} 
\usepackage{subcaption}

\usepackage{multirow}
\usepackage{enumerate}

\usepackage[ruled, vlined, linesnumbered]{algorithm2e}

%% file: commands.tex
\newtheoremstyle{plain}
	  {}
	  {}
	  {\itshape}
	  {}
	  {\bfseries}
	  {}
	  {5pt plus 1pt minus 1pt}
	  {}

\newtheoremstyle{definition}
  	  {}
	  {}
	  {\normalfont}
	  {}
	  {\bfseries}
	  {}
	  {5pt plus 1pt minus 1pt}
	  {}
  
\theoremstyle{plain}

\newtheorem{theorem}{Theorem}
\newtheorem{lemma}{Lemma}
\newtheorem{proposition}{Proposition}
\newtheorem{corollary}{Corollary}

\theoremstyle{definition}
\newtheorem{definition}{Definition}

\newcommand{\refeq}[1]			{(\ref{#1})} 
\newcommand{\reffig}[1]			{Fig. \ref{#1}} 
\newcommand{\refsec}[1]			{Section \ref{#1}}
\newcommand{\refapp}[1]			{Appendix \ref{#1}}

\newcommand{\refprop}[1]		{Proposition \ref{#1}}
\newcommand{\refcor}[1]			{Corollary \ref{#1}}
\newcommand{\reflem}[1]			{Lemma \ref{#1}}
\newcommand{\refdef}[1]			{Definition \ref{#1}}

%% file: notation.tex
\theoremstyle{plain}
\newtheorem*{MotionPlannerExtension}{From Low to High Order Safe Motion Planners}

\newcommand{\R}  	{\mathbb{R}} 
\newcommand{\dimspace} 	{d}
\newcommand{\radius} 	{\rho}

\newcommand{\conv}      {\mathrm{conv}} 

\newcommand{\freespace}	{\mathcal{F}} 
\newcommand{\workspace}	{\mathcal{W}} 
\newcommand{\obstspace}	{\mathcal{O}} 
\newcommand{\ball}      {\vect{B}} 

\newcommand{\path}      {\mathcal{P}}

\newcommand{\vsimplex}{\mathcal{V}}
\newcommand{\simplex}{\vartriangle}

\newcommand{\pos} 		{\vect{x}} 			

\newcommand{\goal}		{\vect{x}^*} 
\newcommand{\ctrlgoal}	{\vect{y}} 

\newcommand{\ctrlmat}{\mat{K}} 

\newcommand{\state}		{\mat{x}} 			
\newcommand{\order}     {n} 				
\newcommand{\gain}   	{\kappa} 			
\newcommand{\ctrl}      {\vect{u}} 			
\newcommand{\lyapmat}	{\mat{P}} 
\newcommand{\decaymat}  {\mat{D}} 

\newcommand{\govpos} 	{\vect{y}}

\newcommand{\govstate}	{\mat{y}}
\newcommand{\govgain}   {\kappa_g}

\newcommand{\refplanner}{\vect{r}}
\newcommand{\refplan}  	{\vect{r}} 
\newcommand{\refdomain}	{\mathcal{D}} 

\newcommand{\motionrange}	{\mathcal{M}} 
\newcommand{\safedist}  {\mathrm{dist}}
\newcommand{\safelevel}{\sigma} 
\newcommand{\motionrangebound}	{\eta} 

\newcommand{\PDM} 	{S_{++}} 
\newcommand{\PSDM} 	{S_{+}} 

\newcommand{\Lyap} 				{\mat{P}}

\newcommand{\proj} {\Pi}

\newcommand{\elp}		{\mathcal{E}} 
\newcommand{\elpctr}	{\vect{c}} 
\newcommand{\elpmat} 	{\mat{\Sigma}} 
\newcommand{\elprad}	{\rho} 

\let\originalleft\left
\let\originalright\right
\renewcommand{\left}{\mathopen{}\mathclose\bgroup\originalleft}
\renewcommand{\right}{\aftergroup\egroup\originalright}
	
\newcommand{\plist}[1] 	{\left(#1\right)} 
\newcommand{\blist}[1]	{\left[ #1 \right]} 
\newcommand{\clist}[1]	{\left\{#1\right\}} 

\newcommand{\vect}[1]   {\mathrm{#1}}
\newcommand{\mat}[1]    {\mathbf{#1}}
\newcommand{\tr}[1] {{#1}^{\mathrm{T}}} 
\newcommand{\norm}[1]  {\|#1\|}
\newcommand{\absval}[1]{\left|#1 \right|} 
\newcommand{\diag} {\mathrm{diag}} 

\newcommand{\ldf}   {:=} 

\newcommand{\argmin}{\operatornamewithlimits{arg\ min}} 
\newcommand{\diff} {\mathrm{d}}
\newcommand{\deriv}	{\mathrm{d}} 